\documentclass[11pt]{article}
\usepackage[utf8]{inputenc}
\usepackage[letterpaper, margin=1in]{geometry}
\usepackage{graphicx}
\usepackage{amsmath}
\usepackage{amsthm}
\usepackage[colorlinks=true,citecolor=blue,linkcolor=red,filecolor=magenta,urlcolor=cyan]{hyperref}
\usepackage{latexsym}
\usepackage{amsfonts}
\usepackage{algorithm}
\usepackage{algorithmic}
\usepackage{color}
\usepackage{dsfont}
\usepackage{amssymb}
\usepackage{natbib}

\newtheorem{lemma}{Lemma}
\newtheorem{corollary}{Corollary}
\newtheorem{theorem}{Theorem}

\newtheorem{definition}{Definition}
\newtheorem{proposition}{Proposition}

\newtheorem{fact}{Fact}

\newcommand\bx{\boldsymbol{x}}
\newcommand\by{\boldsymbol{y}}
\newcommand\bz{\boldsymbol{z}}
\newcommand\btheta{\boldsymbol{\theta}}
\newcommand\bu{\boldsymbol{u}}
\newcommand\bv{\boldsymbol{v}}

\newcommand\E{\mathop{\mathbb{E}}}

\allowdisplaybreaks

\author{
Zihan Zhang\\
Tsinghua University \\
\texttt{zihan-zh17@mails.tsinghua.edu.cn} \\
\and
Xiangyang Ji \\
Tsinghua University \\
\texttt{xyji@tsinghua.edu.cn}
\and
Yuan Zhou\footnote{Corresponding Author.} \\ 
Tsinghua University \\
\texttt{yuan-zhou@tsinghua.edu.cn}
}
\title{Almost Optimal Batch-Regret Tradeoff for Batch Linear Contextual Bandits}

\begin{document}
\maketitle

\begin{abstract}
We study the optimal batch-regret tradeoff for batch linear contextual bandits. For both \emph{context-blind} and \emph{context-aware} settings, we design batch learning algorithms and prove that they achieve the optimal regret bounds (up to logarithmic factors) for
any batch number $M$, number of actions $K$, time horizon $T$, and dimension $d$. Therefore, we establish the \emph{full-parameter-range} (almost) optimal batch-regret tradeoff for the batch linear contextual bandit problem. 

Compared to the recent work \citep{ruan2020linear} which showed that $M = O(\log \log T)$ batches (in the context-blind setting) suffice to achieve the asymptotically minimax-optimal regret without the batch constraints, our algorithm is simpler and easier for practical implementation. Furthermore, our algorithm achieves the optimal regret for all $T \geq d$, while \citep{ruan2020linear} requires that $T$ greater than an unrealistically large polynomial of $d$.

Along our analysis, we also prove a new matrix concentration inequality with dependence on their dynamic upper bounds, which, to the best of our knowledge, is the first of its kind in literature and maybe of independent interest.

\end{abstract}

\section{Introduction}

Online learning and decision-making is an important aspect of machine learning. In contrast to the traditional batch machine learning where the learner only passively observes the data, an online learner may interact with the data collection process by deciding on which data point to query about. On one hand, sequentially making active queries may fully utilize the power of adaptivity based on the observed data and help to achieve better data efficiency. On the other hand, in many practical scenarios, it is also desirable to limit these queries to a small number of rounds of interaction, which helps to increase the parallelism of the learning process, and reduce the management cost and the total time span. In light of this, the \emph{batch online learning} model, which is a combination of the two major aspects of machine learning, has recently attracted much research attention. It has been shown that for many popular online learning tasks, a very small number of batches may achieve nearly minimax-optimal learning performance, and therefore it is possible to enjoy the benefits of both adaptivity and parallelism. 

To understand the impact of the batch constraint to online learning and decision-making problems, in this paper, we study the optimal batch algorithms for the \emph{linear contextual bandit} problem, where the latter is a central problem in online learning literature. In a linear contextual problem, the learning algorithm observes a \emph{context} (also referred to as a \emph{context set} as we usually have one context vector for each candidate action) at the beginning of each time period, and the expected reward of each candidate action is determined by a hidden linear function of the context. The learning algorithm has to learn the linear function and maximize its total reward overall all time periods. The readers may refer to Section~\ref{sec:pre} for the detailed problem definition. The linear contextual bandit problem is widely studied due to its simplicity and abstraction (via the context) for the personalized treatment in decision-making, enabling plenty of real-world applications such as advertisement selection, recommendation systems, and clinical trials.

\paragraph{Context-blind Batch Learning.} One practical reason that calls for the batch online learning model is due to the \emph{expensive policy deployment and communication cost}. Large-scale online advertisement or recommendation systems \citep{li2010contextual} may have to take a long time or pay an expensive overhead cost to update their policies. In many distributed or offline scenarios (such as designing and updating the policy for autonomous vehicles or robotic arms), such an overhead cost becomes even more significant. During the execution of the policy, we usually may not be able to monitor the process (e.g., the data collected from the environment) in real time, due to the high communication cost and latency in the distributed (or large-scale) system. In such cases, the learning algorithm seeks to minimize the number of policy updates. In the batch online learning model, we refer to the time of each policy update as the beginning of a \emph{batch}. The learning algorithm may decide the policy only based on the data collected from the previous batches, and the data collected from the current batch may only be available at the end of the batch. Due to the non-real-time natural, the learning algorithm also has to decide the length of each batch when the batch begins, without any information from the current batch. 

We formalize the above intuition and define the  context-blind batch complexity of the linear contextual bandit problem as follows. Note that the \emph{context-blind setting} means that the policy of each batch may not depend on the contexts observed in the same batch, due to the communication cost and latency.

\begin{definition}[Batch complexity in the context-blind setting]\label{def_batch_RL} 
We first formally define that a \emph{policy} $\pi$ is a mapping from the context space to the set of the distributions over all candidate actions. To \emph{execute} a policy $\pi$ at time step $t$ means to randomly sample and commit to an action from the distribution $\pi(X_t)$ where $X_t$ is the context observed at time $t$.

For a linear contextual bandit problem with time horizon $T$, we say that the \emph{batch complexity} of a learning algorithm is (at most) $M$ in the \emph{context-blind setting}, if the learner decides $T_1$ and $\pi_1$ before the learning process starts, and executes $\pi_1$ in the first $T_1$ time steps (which corresponds the first batch). 
Based on the data (context sets, played actions and the rewards) obtained from the first $T_1$ steps, the learner then decides $T_2$ and $\pi_2$, and executes $\pi_2$ for $T_2$ time steps (the second batch). The learner repeats the process for $M$ times/batches. In general, at the beginning of the $k$-th batch, the learner decides $T_k$ (the size of the batch) and $\pi_k$ based on the data collected from the first $(k-1)$ batches. The batch sizes should satisfy that $\sum_{k=1}^M T_k = T$. 
\end{definition}

\paragraph{Context-aware Batch Learning.} In many other practical scenarios, the batch online learning model is needed only because of the \emph{delayed reward observation}. A typical example of batch online learning is clinical trials \citep{lei2012smart,almirall2012designing,almirall2014introduction}, as all clinical trials run in phases (which exactly correspond to the batches) in order to parallelize the time-consuming treatment experiments. 
Similar example arises in crowdsourcing where it takes significant time to interact with the crowd and the queries may be aggregated into a small number of batches for the crowd to answer so as to save the total time cost. In these examples, at the beginning of each batch, the learner may additionally observe the contexts of the current batch (in the clinical trial example, these contexts correspond to the information of the patients recruited in the current phase). 
In light of this, we define the context-aware batch complexity where the learner may choose the policy based on this additional context information.

\begin{definition}[Batch complexity in the context-aware setting]\label{def_batch_RL-context-aware} 
For a linear contextual bandit problem with time horizon $T$, we say that the \emph{batch complexity} of a learning algorithm is (at most) $M$ in the \emph{context-aware setting}, if at the beginning each batch $k \in \{1, 2, \dots, M\}$, the learner decides the batch size $T_k$, and then observes the context sets of the future $T_k$ time steps. After that, the learner decides the policy $\pi_k$ based on the data obtained from the previous $(k-1)$ batches and the context sets of the future $T_k$ time steps, and executes $\pi_k$ at all these $T_k$ time steps. The batch sizes should satisfy that $\sum_{k=1}^M T_k = T$. 
\end{definition}

Naturally, there is a tradeoff between the batch complexity and the regret performance in batch online learning. There have been quite a few recent works studying such relationship for multi-armed bandits \citep{perchet2016batched,gao2019batched,Esfandiari_Karbasi_Mehrabian_Mirrokni_2021}. In multi-armed bandits, because there is no context information, the context-blind and context-aware settings defined above are identical. \cite{gao2019batched} proved optimal regret bound for every number of batches and show that $M = O(\log \log T)$ batches suffice to achieve the minimax-optimal regret without the batch constraint.

\cite{han2020sequential,ruan2020linear} recently studied the batch algorithms for linear contextual bandits. While \cite{han2020sequential} studied a special case of the problem where the contexts follow Gaussian-type distributions, \cite{ruan2020linear} provided an algorithm for all context distributions. In particual, \cite{ruan2020linear} showed that in the context-blind setting, $M = O(\log \log T)$ batches suffice to achieve the minimax-optimal regret without the batch constraint, which implies the same asymptotic bound for the context-aware setting.\footnote{Both the $O(\log \log T)$ bound in \citep{ruan2020linear} and our work (as well as \citep{han2020sequential}) focus on the \emph{stochastic-context case} of linear contextual bandits, which is the most technically interesting and practically useful setting of batch linear contextual bandits. Please refer to Section~\ref{sec:related-linear-bandits} for more discussion.}

While the exact batch-regret tradeoff for multi-armed bandits is relatively better understood (and arguably easier to study), the optimal tradeoff curve for linear contextual bandits is more challenging and remains open.

In this work, we address this question on the exact batch-regret curve and prove the optimal regret (up to logarithmic factors) for batch linear contextual bandits for the \emph{full range} of the problem parameters in both context-blind and context-aware settings. Compared to previous work \citep{ruan2020linear}, our algorithm is simpler to describe and therefore easier to be practically implemented. Our algorithm is optimal for all $M \geq 1$ and  works for a much wider range of $T$, while in comparison, \citep{ruan2020linear} requires that $T$ greater than an impractically large polynomial of the dimension parameter $d$. Below, we provide a more concrete summary of our contributions and comparison with the related works.

\subsection{Our Contributions and Technical Ingredients}

\subsubsection{Optimal Context-blind Batch-Regret Tradeoff} 

We first study the optimal batch algorithm for stochastic linear contextual bandits in the context-blind setting.

Recall that \cite{gao2019batched} showed that for multi-armed bandits, the optimal regret using at most $M$ batches is at the order of $T^{\frac{1}{2-2^{-M+1}}}$ (ignoring the polynomial dependence on other problem parameters and the poly-logarithmic dependence on $T$). For batch linear contextual bandits in the context-blind, we establish a similar (but slightly trickier) tradeoff. More specifically, we design Algorithm~\ref{alg:main} (in Section~\ref{sec:alg}), and prove the following regret upper bound.
\begin{theorem}\label{thm:main} Let  $d$ be the dimension of the feature space and $K$ be the number of candidate arms.
For any $T\geq d$ and $M\geq 1$,  Algorithm~\ref{alg:main} may use at most $M$ batches in the context-blind setting and its regret $R_{T}$ is bounded by\footnote{Throughout the paper, the $\tilde{O}(\cdot)$ and $\tilde{\Omega}(\cdot)$ notations hide the logarithmic factors of $T$, $d$, and $K$.} \begin{align}
  R_{T}\leq  \tilde{O}\left(\min\left\{ T^{\frac{1}{2-2^{-M+2}}}d^{\frac{1-2^{-M+2}}{2-2^{-M+2}}},T^{\frac{1}{2-2^{-M+1}}}d^{\frac{1-2^{-M+1}}{2-2^{-M+1}}}\min\{K,d\}^{\frac{2^{-M+1}}{2-2^{-M+1}}}\right\}  \right).\nonumber
\end{align}
\end{theorem}

Note that the regret upper bound in Theorem~\ref{thm:main} takes the minimum between two terms. When $T$ is comparably large (e.g., $T \geq \tilde{\Omega}(d\min\{K,d\}^{2-2^{-M+2}})$, the second term in our upper bound is smaller, and its asymptotic dependence on $T$ matches the optimal bound for batch multi-armed bandits \citep{gao2019batched}.\footnote{However, this does not mean that the batch linear contextual bandit problem is easier than batch multi-armed bandits, as the dependence on $K$ and $d$ is worse.} On the other hand, when $T$ is relatively small (e.g., $T\leq \tilde{O}(d\min\{K,d\}^{2-2^{-M+2}})$), a simpler analysis would kick in to give a better regret bound, which results as the first term.

When there are no constraints on the number of batches, it is well known that the minimax-optimal regret is $\sqrt{dT \log K} \times \mathrm{poly}\log T$ (see, e.g., \citep{dani2008stochastic,chu2011contextual,li2019nearly}). \cite{ruan2020linear} showed that with only $M = \lceil \log \log T\rceil + 1$ batches (the logarithms are of base 2), their batch algorithm may match the regret performance (up to logarithmic factors) as the no-batch-constraint setting. The following simple corollary of our Theorem~\ref{thm:main} recovers the main result of \citep{ruan2020linear}. Moreover, the batch algorithm in \citep{ruan2020linear} only works for $T \geq \tilde{\Omega}(d^{32})$, while in contrast, our Corollary~\ref{corom} works for every $T \geq d$ (and note that $T < d$ is the trivial scenario).

\begin{corollary}\label{corom}
For $T \geq d$ and $M =\left\lceil\log \log (T)\right\rceil+1$, the expected regret of Algorithm~\ref{alg:main} is 
$R_{T} \leq \tilde{O}(\sqrt{Td})$.
\end{corollary}

While the two-phase regret curve in Theorem~\ref{thm:main} may seem completely due to technicality, it surprisingly turns out to be exactly optimal. In Section~\ref{sec:low}, we complement Theorem~\ref{thm:main} with the following lower bound.
\begin{theorem}\label{thm:lb}
Fix any $K\geq 2$, $T\geq d$, and any batch number $M \geq 1$. For any learning algorithm with batch complexity $M$ in the context-blind setting, there exists a linear contextual bandit problem instance with dimension $d$ and $K$ arms, such that the expected regret $R_{T}$ is at least
\[R_{T}\geq \tilde{\Omega}\left( \min\left\{T^{\frac{1}{2-2^{-M+2}}} d^{\frac{1-2^{-M+2}}{2-2^{-M+2}}}, T^{\frac{1}{2-2^{-M+1}}}d^{\frac{1-2^{-M+1}}{2-2^{-M+1}}} \min\{K,d\}^{\frac{2^{-M+1}}{2-2^{-M+1}}}   \right\} \right).\]
\end{theorem}

We note that the above upper and lower bounds match (up to factors logarithmic in $T$, $d$, and $K$) for all $M \leq \lceil \log \log T\rceil$ and all non-trivial parameter settings for $T$, $d$, and $K$. When $M > \lceil \log \log T\rceil$, by Corollary~\ref{corom}, our Algorithm~\ref{alg:main} already achieves the unconstrained minimax-optimal regret (up to logarithmic factors). Therefore, we achieve near-optimal regret bounds for (context-blind) batch linear contextual bandits in the context-blind setting under all non-trivial parameter settings.

We also note that the regret lower bound established by \cite{gao2019batched} for $K$-arm $M$-batch multi-armed bandits is $\Omega(\sqrt{K} T^{\frac{1}{2-2^{-M+1}}})$ (which also matches their algorithm up to logarithmic factors). If we treat multi-armed bandits as a special case of linear contextual bandits with $K = d$ arms with orthogonal features, the lower bound for batch linear contextual bandits implied by \citep{gao2019batched} is weaker than our Theorem~\ref{thm:lb}. This gap demonstrates that intrinsic additional difficulty of the context-blind batch linear contextual bandit problem  when compared to its multi-armed bandit counterpart, marking a separation between the two problems.

\subsubsection{Single-Batch Learning for the Exploration Policy}

At the core of our main algorithm is a new single-batch (i.e., offline) procedure, named {\tt ExplorationPolicy}, to learn an exploration policy to achieve the \emph{distributional G-optimality} in experimental designs. Let $D$ be an unknown distribution over the context sets and given a set of $m$ independent samples drawn from $D$, the goal of our {\tt ExplorationPolicy} is to compute a policy $\pi$, so that if one uses $\pi$ to collect $n$ more data points\footnote{Here, given a context $X \sim D$, a data point is collected by playing the action $\pi(X)$ and observing the reward. Please refer to Section~\ref{sec:pre} for the detailed formulation of linear contextual bandits.} and estimate the underlying linear model $\btheta$, the expected size of the largest confidence interval among the actions in a random context set will be small. Note that if the context set is deterministic, this objective corresponds to (the square root of) the G-optimality criterion in classical experimental designs (see, e.g., \citep{pukelsheim2006optimal,atkinson2007optimum}). For stochastic context sets, the objective was recently found closely related to linear contextual bandits and studied by \citep{ruan2020linear,zanette2021design}.

Our {\tt ExplorationPolicy} procedure is employed during each batch of the main algorithm to decide the policy used in the next batch. This procedure is very similar to Algorithm 2 in \citep{ruan2020linear}. However, the difference is that \cite{ruan2020linear} only used their Algorithm 2 to prove the existence of a good exploration policy and designed several more complicated procedures (such as {\sc CoreLearning} and {\sc CoreIdentification}) for the policy learning. In contrast, thanks to a few new algorithmic techniques (e.g., the \emph{scaled-and-clipped update rule} to be explained in Section~\ref{sec:tec-ub-overview}) and a new matrix concentration inequality (to be introduced in Section~\ref{intro:concentration}), our algorithm can directly learn the desired exploration policy with better performance, and is simpler to describe and implement.

We also note that in the concurrent work \citep{zanette2021design}, the authors studied a similar task to our {\tt ExplorationPolicy}. In Section~\ref{sec:tec-ub-overview}, we compare our performance guarantee and the results in \citep{zanette2021design}, and demonstrate the superiority of our procedure.

\subsubsection{A New Matrix Concentration Inequality with Dynamic Upper Bounds} \label{intro:concentration}

Existing matrix concentration inequalities (see, e.g., \citep{tropp2012user}) play an important role in recent works on batch linear contextual bandits \citep{ruan2020linear,zanette2021design}. For example, the proof techniques of Theorem 5.1 in \citep{tropp2012user} may yield the following concentration bound in Proposition~\ref{prop:tropp} (and the upper bound on $\sum_{k} X_k$ may be similarly derived). Special cases of Proposition~\ref{prop:tropp} (taking $W = \mathbf{I}$ and $\epsilon = \mathrm{const.}$) includes Lemma 21 in \citep{ruan2020linear} and Lemmas 11 \& 12 in \citep{zanette2021design}.

\begin{proposition}\label{prop:tropp} Consider a sequence of independent PSD matrices $X_1,X_2,\ldots,X_n\in \mathbb{R}^{d\times d}$ such that $X_{k}\preccurlyeq W$ for a fixed PSD matrix $W$ and all $1\leq k \leq n$.
There exists a universal constant $c > 0$ such that for every $\delta>0$ and $\epsilon\in (0,1)$, it holds that
\begin{align}
\Pr\left[\sum_{k=1}^n X_k \preccurlyeq (1+\epsilon)\sum_{k=1}^{n} \mathbb{E}[X_k] + \frac{c \ln(d/\delta)}{\epsilon} W\right] \geq 1 - \delta;  \\
\Pr\left[\sum_{k=1}^n X_k \succcurlyeq (1-\epsilon)\sum_{k=1}^{n} \mathbb{E}[X_k] - \frac{c \ln(d/\delta)}{\epsilon} W\right] \geq 1 - \delta.
\end{align}
\end{proposition}





However, to achieve the optimal batch-regret tradeoff in the context-blind setting, we need a stronger version of Proposition~\ref{prop:tropp} where the uniform upper bound matrix $W$ may be stochastic. In particular, we prove the following lemma as a crucial technical tool in our algorithm analysis (especially for the {\tt ExplorationPolicy} procedure). 

\begin{lemma}\label{lemma:dynamic-martingale-concentration}
Consider a sequence of stochastic PSD matrices $W_1, X_1, W_2, X_2, \dots, W_n, X_n \in \mathbb{R}^{d\times d}$. Let $\mathcal{F}_k = \sigma(W_1, X_1, W_2, X_2, \dots, W_{k-1}, X_{k-1})$ and $\mathcal{F}_k^+ = \sigma(W_1, X_1, W_2, X_2, \dots, W_{k-1}, X_{k-1}, W_k)$ be the natural filtration and $Y_k = \mathbb{E}[X_k |\mathcal{F}_k^+]$ for each $k \in \{1, 2, \dots, n\}$. Suppose $W_k$ is PD and increasing in $k$ (with respect to the semidefinite order) and $X_k \preccurlyeq  W_k$ for each $k$. For every $\delta > 0$ and $\epsilon\in (0,1)$, we have that
\begin{align}
& \Pr\left[\sum_{k=1}^n X_k \preccurlyeq (1+\epsilon)\sum_{k=1}^{n} Y_k + \frac{4(\epsilon^2+2\epsilon+2)}{\epsilon}\ln((n+1)d/\delta) W_n\right] \geq 1 - \delta; \label{eq:con1}
 \\ & \Pr\left[\sum_{k=1}^n X_k \succcurlyeq (1-\epsilon)\sum_{k=1}^{n} Y_k - \frac{4(\epsilon^2+2\epsilon+2)}{\epsilon}\ln((n+1)d/\delta) W_n\right] \geq 1 - \delta.\label{eq:con2}
\end{align}
\end{lemma}
In Lemma~\ref{lemma:dynamic-martingale-concentration}, the stochastic matrix $W_n$ upper bounds all matrices in $\{X_1, X_2, \dots, X_n\}$. When it is fixed, the lemma reduces to Proposition~\ref{prop:tropp}.\footnote{Indeed, we lose an additional $\ln(n+1)$ term in the $\ln(d/\delta)$ terms in Proposition~\ref{prop:tropp}, and we do not know if this compromise is necessary.} There are also a few Freedman's inequalities for matrix martingales (see, e.g., \citep{tropp2011freedman}). However, in these inequalities, while the quadratic variation $\mathbb{E}[X_i^2 | \mathcal{F}_i^+]$ becomes dynamic, the uniform upper bound $W_k$ is still fixed.

We note that if an extra $\mathrm{poly}(d)$ factor were allowed in the $\pm \frac{4(\epsilon^2+2\epsilon+2)}{\epsilon}\ln((n+1)d/\delta) W_n$ terms, the lemma would easily follow from Proposition~\ref{prop:tropp} and an $\epsilon$-net argument. However, reducing these $\mathrm{poly}(d)$ factors is crucial to the \emph{full parameter range} optimality analysis of our batch algorithm.

 The full proof of Lemma~\ref{lemma:dynamic-martingale-concentration} is presented in Section~\ref{sec:pfdy}.
We believe that Lemma~\ref{lemma:dynamic-martingale-concentration} is a non-trivial addition to the vast family of matrix concentration inequalities and may be of its own interest.

\subsubsection{Optimal Context-aware Batch-Regret Tradeoff}

Note that any context-blind batch learning algorithm also meets the definition of the context-aware algorithm with the same batch complexity. However, the learning algorithm may take the advantage of the additional context information in each batch to improve the regret. Using the same techniques developed above, we prove the following regret upper bound for context-aware batch learning in Section~\ref{sec:context-aware}. 

\begin{theorem}\label{thm:main-context-aware}
In the context-aware batch learning setting, for any $T\geq d$ and $M\geq 1$, Algorithm~\ref{alg:context-aware} uses at most $M$ batches and its regret $R_{T}$ is bounded by \begin{align*}
R_{T}\leq  \tilde{O}\left(T^{\frac{1}{2-2^{-M+1}}} d^{\frac{1-2^{-M+1}}{2-2^{-M+1}}} \right) .
\end{align*}
\end{theorem}

The regret upper bound in Theorem~\ref{thm:main-context-aware} is achieved by slightly adjusting our optimal algorithm for the context-blind case as well as the batch sizes. Also in Section~\ref{sec:context-aware}, we prove that the batch-regret trade-off achieved in Theorem~\ref{thm:main-context-aware} is optimal.
\begin{theorem}\label{thm:lb-context-aware}
Fix any $K\geq 2$, $T\geq d$, and any batch number $M \geq 1$. For any learning algorithm with batch complexity $M$ in the context-aware setting, there exists a linear contextual bandit problem instance with dimension $d$ and $K$ arms, such that the expected regret $R_{T}$ is at least
\[R_{T}\geq \tilde\Omega\left(T^{\frac{1}{2-2^{-M+1}}} d^{\frac{1-2^{-M+1}}{2-2^{-M+1}}} \right).\]
\end{theorem}

Theorem~\ref{thm:lb-context-aware} is proved by a simple adaptation from the proof of the lower bound theorem for the context-blind setting (Theorem~\ref{thm:lb}). Comparing our bound with the optimal regret bound  $\tilde{\Theta}(\sqrt{K} T^{\frac{1}{2-2^{-M+1}}})$ for $K$-arm $M$-batch multi-armed bandits  \citep{gao2019batched}, we see that the order on $T$ are the same. The different dependence on $K$ and $d$ is due to the slightly different problem setting -- in \citep{gao2019batched}, the expected rewards of the arms may be as large as $\sqrt{K}$ and in our paper we assume that they are bounded by $[-1, 1]$.

\subsection{Related Works}

\subsubsection{Linear Contextual Bandits} \label{sec:related-linear-bandits}

The linear contextual bandit problem
\citep{abe1999associative,auer2002finite} studies the bandit problem where the actions are associated with (known) features and their mean rewards are defined by an (unknown) linear function of the associated features.  Compared with the multi-armed bandit problem, the linear structure on features could help the learner to infer the mean reward of an action given the observation on the other actions, and therefore enables the possibility to achieve regret upper bounds independent from (or weakly dependent on) the number of actions. 

There are generally two types of problem settings studied about linear contextual bandits: \emph{non-adaptive contexts} and \emph{adaptive contexts}. In the non-adaptive-context setting, the context sets are independent from all other randomnesses (including the randomnesses in rewards and used by the algorithm). One can also think of this as that the contexts are fixed (by an adversary) before the learning process starts.  In this setting, the optimal minimax regret bound is $\Theta(\sqrt{dT\min\{d,\ln(K)\}})$ up to $\mathrm{poly} \ln(T)$ factors \citep{auer2002finite,abe2003reinforcement,dani2008stochastic,chu2011contextual,abbasi2011improved,li2019nearly}. In the adaptive-context setting, the context sets are chosen by an adaptive adversarial, where the context sets at any time step may depend on the outcomes and the learner's decisions in previous time steps. In this setting, the problem becomes harder for the learner. To the best of our knowledge, the state-of-the-art regret upper bound for the adaptive-context setting is $O(
d\sqrt{T\ln(KT)})$ \citep{abbasi2011improved}.

In this work (as well as the most related works \citep{han2020sequential,ruan2020linear,zanette2021design} on batch linear contextual bandits), we focus on a particularly useful case in the non-adaptive-context setting, namely the \emph{stochastic contexts}. In this case, the context sets at each time step are independently generated from a pre-defined (but unknown) distribution $D$. In many real-world applications such as clinical trial and recommendation system, the patients or customers can often be viewed as independent samples from the population and therefore stochastic contexts are a natural abstraction of these practical scenarios. On the other hand, \cite{han2020sequential} has shown that even in the non-adaptive-context setting, in the worst case, as many as $\Omega(\sqrt{T})$ batches are needed to achieve any $\sqrt{T}$-type regret, which is less useful in practice.

\subsubsection{Bandit Learning with Limited Adaptivity}

Batch learning fits into the broader \emph{learning with limited adaptivity} framework that recently attracts much research attention due to its potentially lower computational cost and close relation to distributed and parallel learning. 

The number of batches is a natural measurement of the adaptivity needed by the learner. Besides the above mentioned works \citep{perchet2016batched,gao2019batched} (for batch multi-armed bandits) and \citep{ruan2020linear} (for batch linear contextual bandits), \cite{han2020sequential} studied batch linear contextual bandits with Gaussian-type features and \cite{esfandiari2019regret} studied batch adversarial multi-armed bandits.



Besides batch learning, another type of adaptivity measurement studied in literature is the policy switching cost, where the learner may monitor the sequential decisions but would like to change his/her decision policy as infrequently as possible. This is a comparably more lenient constraint, as a batch algorithm has a small number of policy switches equal to the number of batches.  For the adversarial multi-armed bandit problem, \cite{kalai2005efficient} and \cite{geulen2010regret} established a minimax regret bound of $\tilde{\Theta}(\sqrt{T})$ under the full information feedback; \cite{dekel2014bandits} later showed that the minimax regret bound is $\Theta(T^{\frac{2}{3}})$  under the bandit feedback. \cite{simchi2019phase} studied the switching cost of stochastic multi-armed bandits. For linear contextual bandits, it was shown that to achieve the $\sqrt{T}$-type regret, the optimal bound for switching cost is $d \log T$, up to  $\mathrm{poly}\log(d \log T)$ factors \citep{abbasi2011improved,ruan2020linear}.


\subsection{Organization}
The rest of the paper is organized as follows. In Section~\ref{sec:pre} we introduce the batch linear bandit problem and some notations. 
Then we introduce our main technical contributions and proof ideas in Section~\ref{sec:tec}. 
Section~\ref{sec:alg} is devoted to describing our main algorithm for the context-blind setting and the proof of its regret analysis (Theorem~\ref{thm:main}). In Section~\ref{sec:exppolicy}, we describe the key procedure {\tt ExplorationPolicy} and present its analysis. In Section~\ref{sec:pfdy}, we prove our new matrix concentration inequality (Lemma~\ref{lemma:dynamic-martingale-concentration}). In Section~\ref{sec:low}, we present the proof of the lower bound for the context-blind setting (Theorem~\ref{thm:lb}). In Section~\ref{sec:context-aware}, we extend our results and prove the regret upper and lower bounds for the context-aware setting (Theorem~\ref{thm:main-context-aware} and Theorem~\ref{thm:lb-context-aware}). Finally, we conclude the paper in Section~\ref{sec:conclusion}.

\section{Preliminaries}\label{sec:pre}

\paragraph{Linear Contextual Bandits with Stochastic Context Sets.} 
We consider the linear contextual bandit problem with the hidden linear model described by the $d$-dimensional vector $\btheta: \|\btheta\|_{\infty}\leq 1$. There is also a distribution $D$ over the context sets hidden from the learner. Given the time horizon $T$, during each time step $t \in \{1, 2, \dots, T\}$, a stochastic context set of $K$ feature vectors, $X_t =\{\bx_{t,1},\bx_{t,2},\dots,\bx_{t,K}\}$ is drawn from $D$ and revealed to the learner. The feature vectors are in $\mathbb{R}^d$ and $D$ guarantees that $\forall i \in \{1, 2, \dots, K\}: |\bx_{t, i}^\top \theta| \leq 1$ almost surely.\footnote{Note that our formulation is more general than the usual linear contextual bandits setting where $\|\btheta\|_2 \leq 1$ and $\|\bx_{t, i}\|_2 \leq 1$. It also includes $K$-armed multi-armed bandits as a special case.}

The learner has to choose and play an action (defined by its associated feature vector) $\by_t \in X_t$ and receives the reward $r_t = \by_t^\top \btheta + \epsilon_t$, where $\epsilon_t$ is an independent sub-Gaussian noise with zero mean and variance proxy bounded by $1$. The goal of the learner is to minimize the total (expected) regret
\[
R_{T}:= \mathbb{E}\left[ \sum_{t=1}^{T}\left( \max_{i}\{\bx_{t,i}^{\top}\btheta\} -\by_t^{\top}\btheta \right) \right]. 
\]



\paragraph{Batch Learning.} Given the batch complexity $M$, a batch learning algorithm aims at minimize the regret $R_T$ defined above, subject to the constraints in Definition~\ref{def_batch_RL} and Definition~\ref{def_batch_RL-context-aware} for the context-blind and context-aware settings respectively.

\paragraph{Notations.} 
For any non-negative integer $N$, we let $[N]$ denote $\{1,2,\ldots,N\}$.  We use $\mathbb{E}_{\mathcal{P}}[\cdot]$ and $\mathrm{Pr}_{\mathcal{P}}[\cdot]$ to denote the expectation and probability over the distribution $\mathcal{P}$ respectively. We use $\mathbf{I}$ to denote the $d$-dimensional identity matrix. We use $\log$ to denote the logarithm base $2$, and use $\ln$ to denote the logarithm base $e$.  We also define $\mathcal{T}_0 = T_0 = 0$ and $\mathcal{T}_k = \sum_{i=1}^{k}T_i$ for $1\leq k \leq M$.

\section{Technical Overviews}\label{sec:tec}

\subsection{Upper Bounds} \label{sec:tec-ub-overview}
Since our algorithms for the both context-blind and context-aware settings are similar and adopt the same  techniques, we present the technical overview based on the context-blind version. Our algorithms are elimination-based, following \citep{ruan2020linear}. At each time step $t$, give the set of context vectors $X_t=\{\bx_{t,1},\bx_{t,2},\ldots,\bx_{t,K}\}$, we maintain a confidence interval $\mathcal{I}_{t,i}$ for $\bx_{t,i}^{\top}\theta$ for each $i$. A candidate action $\bx_{t,i}$ is eliminated when there exists another candidate action $\bx_{t,i'}$ such that $\mathcal{I}_{t,i'}$ entirely lies above $\mathcal{I}_{t,i}$, meaning that the action $\bx_{t,i}$ cannot be the optimal action. Then, the clever part of the policy is to decide a distribution over the remaining candidate actions and randomly choose one to commit to according to the distribution. 

For the construction of the confidence intervals, we adopt the classical elliptical confidence intervals based on the regularized ordinary least-square (OLS) estimation~\citep{chu2011contextual}. Given a group of context vectors that are played in history $\{\by_\tau\}_{\tau=1}^t$ and corresponding observed rewards $\{r_\tau\}_{\tau=1}^t$ (such that $r_\tau = \by_\tau^{\top}\btheta + \epsilon_\tau$ where $\epsilon_\tau$ is an $1$-subgaussian noise), we construct the confidence interval for any candidate action with context vector $\bx$ to be 
\begin{align}
\mathcal{I}(\bx,\Lambda)=\left[\bx^{\top}\hat{\btheta}-\alpha\sqrt{\bx^{\top}\Lambda^{-1}\bx},\bx^{\top}\hat{\btheta}+\alpha\sqrt{\bx^{\top}\Lambda^{-1}\bx}\right]\bigcap [-1,1],
\end{align}
with $\Lambda = \lambda \mathbf{I} + \sum_{\tau=1}^t \by_\tau \by_\tau^{\top}$ is the \emph{regularized information matrix}, $\hat{\btheta}=\Lambda^{-1}\sum_{\tau=1}^t r_\tau\by_\tau$ is the regularized OLS estimation of the hidden vector $\btheta$, and  $\alpha,\lambda$ are hyper-parameters satisfying that $\alpha =\Theta(\sqrt{\ln(KdT)}+\lambda \sqrt{d})$. 

Define 
\[
w(\bx,\Lambda) =\min\left\{  ( \sqrt{\ln(KdT)}+\lambda \sqrt{d})\sqrt{\bx^{\top}\Lambda^{-1}\bx},1\right \}
\]
to be the \emph{width} of the confidence interval $\mathcal{I}(\bx,\Lambda)$. To reduce the regret, we would like to design policies to cleverly perform exploration in order to reduce the width of future estimations. 
Formally, we introduce the following problem which is the key to our optimal batch learning algorithm.

\medskip
\noindent {\bf The Problem of Single-Phase Learning for Exploration Policy.}
{\it 
Fix $m,n \geq 0$. Let $\{X_i\}_{i=1}^m, \{Y_{j}\}_{j=1}^n$ be two groups of \emph{i.i.d.}~context sets following the same unknown distribution $D$. After observing $\{X_i\}_{i=1}^m$, we are asked to design the parameter $\lambda$ and an exploration policy $\pi$ so as to minimize the following expected maximum width
\[
\text{\rm EM-width}(D,\Lambda) := \E_{X\sim D}[\max_{\bx\in X} w(\bx,\Lambda)],
\]
where $\Lambda = \lambda \mathbf{I} + \sum_{j=1}^n \by_{j}\by_j^{\top}$ and $\by_j \sim \pi(Y_j)$ for all $1\leq j \leq n$. 
}
\medskip

In our batch learning algorithm, we need to solve the above problem once during each batch. During the $k$-th batch, we let $D = D_k$ be the distribution of the set of the remaining context vectors after the elimination process based on the information learned for the first $(k-1)$ batches. $\{X_i\}_{i=1}^m$ is obtained from the $(k-1)$-th batch, and we solve the above problem for $\pi$ which serves as the exploration policy for the $k$-th batch. The minimization goal of the above problem helps to reduce the regret starting from the $(k+1)$-th batch.

To facilitate discussion, we use $\mathcal{W}(m,n)$ to denote the minimax optimum of expected maximum width achieved the best learning algorithm $\mathcal{G}$. That is, we let 
\[
\mathcal{W}(m,n) = \inf_{\mathcal{G}}\sup_{D}\E_{\{X_i, Y_{j}\} \sim D^{\otimes (m+n)}} \E_{(\pi, \lambda) \sim \mathcal{G}(\{X_i\})} \left[\text{\rm EM-width}(D,\Lambda)\right],
\]
where $\mathcal{G}$ is the single-phase learning algorithm to decide $\pi$ and $\lambda$ based on $\{X_i\}_{i=1}^m$.



In \citep{ruan2020linear}, the authors showed the existence of a good policy $\pi$ and the choice of $\lambda$ so that the bound of the expected maximum width leads to the desired optimal regret for $M = \Theta(\log \log T)$ batches. Their constructive proof (given the distribution $D$) is based on a reward-free LinUCB algorithm (Algorithm 2 in their paper). However, to learn such a good policy based on $\{X_i\}_{i=1}^m$, the authors employed more complicated procedures (such as {\sc CoreLearning} and {\sc CoreIdentification}). 

Both \citep{zanette2021design} and our work are inspired by the reward-free LinUCB and find that one may leverage this algorithmic framework to design learning algorithms as well. The authors of \citep{zanette2021design} worked on a similar task as the single-phase learning problem defined above and their result implies that $\mathcal{W}(m,n)\leq  O\left( \mathrm{poly}\ln(mndT)\cdot \sqrt{\ln(K)} \cdot(\sqrt{d/n}+\sqrt{d/m})\right)$.\footnote{
We state this implication by making the ``large $|\mathcal{S} \times \mathcal{A}|$'' assumption in \citep{zanette2021design}. 
} 

In comparison, in Section~\ref{sec:exppolicy} we propose {\tt ExplorationPolicy} to solve the single-phase learning problem. In Lemma~\ref{lemma:design} we analyze our algorithm and show that 
\begin{align}\label{eq:overview-elimination-exploration-policy-1}
\mathcal{W}(m,n) \leq  O\left( \mathrm{poly}\ln(mndT)\cdot  \sqrt{\ln(K)}\cdot (\sqrt{d/n} +d/m)\right).
\end{align}
Clearly, the performance of our {\tt ExplorationPolicy} outperforms that the results in \citep{zanette2021design} in terms of the dependence on $m$. 
Note that in our batch learning algorithm, $m$ represents the number of samples in the previous batch, which is much smaller than the size of the current batch (represented by $n$). Therefore, the $d/m$ term in our bound usually dominates and it is crucial for us to make this $\sqrt{m}$-factor improvement to achieve the optimal regret in the batch learning model. Indeed, without this improvement, the result of \citep{zanette2021design} does not even imply the desired optimal regret for $M = \Theta(\log \log T)$ batches (the result of \citep{ruan2020linear}).

The proof of Lemma~\ref{lemma:design} is based on the analyais of reward-free LinUCB, and involves a scaled-and-clipped update rule and a dynamic concentration inequality for PSD matrices. Below we present the high-level ideas.

\paragraph{Learning the Exploration Policy via Reward-free LinUCB.} In \citep{ruan2020linear}, the authors showed that given $\{Y_j\}_{j=1}^n \sim D^{\otimes n}$, the reward-free LinUCB algorithm can produce $\{\by_{j}\}_{j=1}^n$ (and therefore also form a policy) such that 
\begin{align}\label{eq:overview-reward-free-linucb-0.5}
\text{\rm EM-width}(D, U)\leq O(\sqrt{d\ln(nd/\kappa)/n}),    
\end{align} 
where $U =\kappa \mathbf{I} + \sum_{j=1}^n \by_{j} \by_j^{\top}$, and $\kappa > 0$ is polynomially small (e.g., $\kappa = T^{-2}$) to make sure that $U$ is invertible while we do not lose much in \eqref{eq:overview-reward-free-linucb-0.5}. In this work, our {\tt ExplorationPolicy} algorithm cannot direct access $\{Y_j\}_{j=1}^n$ but has to learn the distribution $D$ and construct a policy $\pi$ based on a much smaller data set $\{X_i\}_{i=1}^m$. \footnote{Comparing \eqref{eq:overview-elimination-exploration-policy-1} and \eqref{eq:overview-reward-free-linucb-0.5}, we also find that the cost we pay in the expected maximum width for learning is about $\tilde{O}(d/m)$.}

To make learning possible, we first notice that the original reward-free LinUCB in \citep{ruan2020linear} produces $\{\by_{j}\}_{j=1}^n$ by the so-called \emph{argmax policy}: $\by_j = \pi_j(Y_j) := \arg\max_{\by  \in Y_j} \{\by^\top W_j \by\}$, where $W_j = \kappa \mathbf{I} + \sum_{q=1}^{j} \by_q \by_q^\top$  is the regularized information matrix obtained from the samples before $j$.\footnote{We warn the readers that this is an oversimplification of the algorithm by omitting a few important techniques such as the volume-based lazy update of the $W_j$ matrices. However, we choose to the current presentation to better motivate our technical contributions.} One may combine $\{\pi_j\}$ via carefully chosen probability weights to form a desired one-shot policy $\pi$ (the \emph{mixed argmax policy}).  

In our algorithm {\tt ExplorationPolicy}, we observe that we may approximately learn the policy $\pi$ from $\{X_i\}_{i=1}^m$ as long as we are able to approximately construct $\{W_j\}$ based on the $\{X_i\}$ data set. I.e., for any $W = W_j$, we would like to construct $\check{W}$ as long as $\Omega(1) \cdot \check{W} \preccurlyeq W \preccurlyeq \tilde{O}(1) \cdot \check{W}$, and the key here is to lower bound $W$ by $\Omega(1) \cdot \check{W}$.

To illustrate the main technical challenge and our solution, let us consider the following task: let $\bx_i \sim \pi(X_i)$ and $\check{W} =\frac{n}{m} (\kappa \mathbf{I} + \sum_{i=1}^m \bx_i \bx_i^\top)$, we would like to choose appropriate regularization parameter $\lambda > 0$ so that with high probability (over the randomness of $\{X_i\}$ and $\{\bx_i\}$) it holds that
\begin{align}\label{eq:overview-reward-free-linucb-1}
\Omega(1) \cdot \check{W} \preccurlyeq n \left(\lambda \textbf{I} + \E_{X \sim D, \bx \sim \pi(X)} \bx \bx^\top \right) .
\end{align}
We note that this task may seem a bit different from our goal: 1) $\pi$ is unknown to the learner, and cannot be used to construct $\check{W}$; 2) the upper bound is quite different from $W = \kappa \textbf{I} + \sum_{j=1}^{n} \by_j \by_j^\top$. Indeed, these issues may be (quite non-trivially) resolved by observing that 1) $\pi$ is a mixed policy and can be iteratively updated to its final form and 2) relate the Right-Hand-Side of \eqref{eq:overview-reward-free-linucb-1} to $W$ by another matrix concentration inequality. 

We now focus on the task of \eqref{eq:overview-reward-free-linucb-1}, which is equivalent to 
\begin{align}\label{eq:overview-reward-free-linucb-2}
\Omega(1) \cdot \left( \kappa \mathbf{I} + \sum_{i=1}^m \bx_i \bx_i^\top\right)  \preccurlyeq m \left(\lambda \textbf{I} + \E_{X \sim D, \bx \sim \pi(X)} \bx \bx^\top \right) .
\end{align}
Note that since $\bx_i \sim \pi(X_i)$ are \emph{i.i.d.}~random variables, standard matrix concentration inequalities would imply \eqref{eq:overview-reward-free-linucb-2} when $m \lambda \mathbf{I}$ upper bounds $\bx_i \bx_i^\top$ (up to logarithmic factors of the inverse of the failure probability) almost surely, i.e., $m \lambda \geq \tilde\Omega(1) \Leftrightarrow \lambda \geq \tilde\Omega(1/m)$. This choice of $\lambda$ would lead to a $d/\sqrt{m}$ term instead of the $d/m$ term in \eqref{eq:overview-elimination-exploration-policy-1}.

\paragraph{The Scaled-and-Clipped Update Rule.} While it is not possible to establish \eqref{eq:overview-elimination-exploration-policy-1} (with high probability) for a smaller $\lambda$ (e.g., $\lambda = o(1/m)$), we introduce the \emph{scaled-and-clipped update rule} in the reward-free LinUCB (Line~\ref{line:alg-exppolicy-6} of Algorithm~\ref{alg:exppolicy}) which eventually leads to the improvement of $\lambda$.

More concretely, instead of working with $\check{W} =\frac{n}{m} (\kappa \mathbf{I} + \sum_{i=1}^m \bx_i \bx_i^\top)$, we define the \emph{scaled-and-clipped version} of $\bx_i$ and the \emph{scaled-and-clipped information matrix} $U_i$,\footnote{The definition of $\tilde{\bx}_i$ here is slightly different from the real algorithm (up to a logarithmic factor $L$). We make this simplification only to better explain the main algorithmic ideas.}
\[
\tilde{\bx}_i :=\min\left\{ \sqrt{\frac{1}{\bx_i^{\top}U_{i-1}^{-1}\bx_i}}, 1\right\}\bx_i, \qquad  U_{i} = \kappa \mathbf{I}+\sum_{p=1}^i \tilde{\bx}_p\tilde{\bx}_p^{\top} .
\]
We will use $\{U_i\}$ to construct the mixed argmax policy instead of $W$'s. The downside of this new update rule is that we use shorter feature vectors $\{\tilde{\bx}_i\}$ instead of the original ones, which leads to the slower growth of the information matrix. However, this slowing effect is not too bad -- if we repeat each $\bx_i$ by $\ln (1/\kappa) = O(\ln T)$ times, scaled-and-clipped information matrix $U_i$ would upper bound the original information matrix. Through a more rigorous analysis, we will see that this effect would only hurt the regret by a logarithmic factor.

On the other hand, the benefit of our scaled-and-clipped update rule is that instead of establishing \eqref{eq:overview-reward-free-linucb-2}, we only need to lower bound the Right-Hand-Side of \eqref{eq:overview-reward-free-linucb-2} by the scaled-and-clipped information matrix, i.e., to prove that the following inequality holds with high probability.
\begin{align}\label{eq:overview-clipped-update-1}
\Omega(1) \cdot \left( \kappa \mathbf{I} + \sum_{i=1}^m \tilde{\bx}_i \tilde{\bx}_i^\top\right) \preccurlyeq m \left(\lambda \textbf{I} + \E_{X\sim D, \bx\sim \pi(X)} \bx\bx^\top\right),
\end{align}
where $\tilde{\bx}_i =\min\left\{ \sqrt{\frac{1}{\bx_i^{\top}U_{i-1}^{-1}\bx_i}}, 1\right\}\bx_i$ as before and we assume that $\bx_i \sim \pi(X)$ are \emph{i.i.d.} (as we did in \eqref{eq:overview-reward-free-linucb-2}).

Let $U = \kappa \mathbf{I} + \sum_{i=1}^m \tilde{\bx}_i \tilde{\bx}_i^\top$. To prove \eqref{eq:overview-clipped-update-1}, we only need to show that there exists constants $c_2 > c_1 > 0$ such that
\begin{align}
c_2 U  \preccurlyeq c_1 U +  m \lambda \textbf{I} + m \E_{X\sim D, \bx\sim \pi(X)} \tilde{\bx}\tilde{\bx}^\top , \label{eq:overview-clipped-update-2}
\end{align}
which implies that 
\begin{align*}
(c_2 - c_1) U  \preccurlyeq m\lambda \mathbf{I}  + \E_{X\sim D, \bx\sim \pi(X)} \tilde{\bx}\tilde{\bx}^\top  \preccurlyeq m\lambda \mathbf{I} + \E_{X\sim D, \bx\sim \pi(X)} {\bx}{\bx}^\top , 
\end{align*}
where the last inequality is due to the clipping operation in the update rule.

Now let us focus on the task of establishing \eqref{eq:overview-clipped-update-2}. Thanks  to the definition of the scaled-and-clipped version $\tilde{\bx}_i$, we have that $\tilde{\bx}_i \tilde{\bx}_i^\top \preccurlyeq U$ almost surely. Therefore, it is possible to establish \eqref{eq:overview-clipped-update-2} as long as we choose $m \lambda \geq \Omega(\kappa)$ to cover the $\kappa I$ term in $U$ (which only requires that $\lambda \geq \Omega(\kappa/m) = \Omega(m^{-1} T^{-2})$, leading to the better $d/m$ error term in \eqref{eq:overview-elimination-exploration-policy-1}). On the other hand, however, we note that $U$, while serving as an upper bound of the random matrices $\tilde{\bx}_i \tilde{\bx}_i^\top$, is also a random variable by itself. We do not find sharp matrix concentration inequalities in literature to fit our need, and we have to resort to the matrix concentration inequality with dynamic upper bounds proved in Lemma~\ref{lemma:dynamic-martingale-concentration}. The formal version of \eqref{eq:overview-clipped-update-2} is stated and proved in Lemma~\ref{lemma:emp2}.


\paragraph{Matrix Concentration with Dynamic Upper Bounds.} As stated above, our Lemma~\ref{lemma:dynamic-martingale-concentration} is crucial to the analysis of the scaled-and-clipped update rule. One simple approach to prove Lemma~\ref{lemma:dynamic-martingale-concentration} is to assume the upper bound matrix $W_n$ (correspondingly $U$ in \eqref{eq:overview-clipped-update-2}) were fixed, apply the ordinary matrix concentration inequality and finally take a union bound over an $\epsilon$-net of $W_n$. However, such an approach would introduce extra $\mathrm{poly}(d)$ factors in the bound and lead to sub-optimal regret bound for our batch learning algorithm.

Our proof of Lemma~\ref{lemma:dynamic-martingale-concentration} follows the classical exponential moment method. However, we analyze the exponential moment of a specially chosen matrix, namely $W_n^{-1/2} (\sum_{i=1}^n (X_i - (1+\epsilon) Y_i)) W_n^{-1/2}$. To bound the trace of this exponential moment, we resort to some deep analysis about Lieb's theorem on convex trace functions (Theorem~\ref{thm:lieb} and Lemma~\ref{lemma:accc4}). Besides, we also make the critical observation that the function $\mathrm{Tr}\left(\exp\left(U^{\top}AU\right)\right)$ is bounded by $\mathrm{Tr}(\exp(A))$ with an additive error at most $d$ for any $U$ such that $U^{\top}U\preccurlyeq \mathbf{I}$ and any symmetric $A$ (Lemma~\ref{lemma:accc3}). 


\subsection{Lower Bounds}

\paragraph{Context-blind Lower Bound.}
To construct the lower bound instances, we first construct $M$ mutually independent sub-problems with dimension $d/M$. At each time step, a uniform random sub-problem (i.e., its context set) is selected and presented to the learner. To analyze the regret performance of any batch learning algorithm, we divide the time horizon $T$ into $M$ consecutive stages with properly chosen lengths $\{T_k\}_{k=1}^M$ for each stage. We will show that, for each $k \in \{1, 2, \dots, M\}$, during the $k$-th stage, if the learner does not start a new batch and update its policy, then a large regret would incur for the $k$-th sub-problem in this stage. 

Our construction for the sub-problems is as follows.
 For each $k \in \{1, 3, 4, 5, \dots, M\}$, we choose the sub-problem to be the hard instance for the $K$-armed linear contextual linear bandit problem with burn-in time $\tilde{\Theta}(T_{k-1})$. In other words, we construct the $k$-th sub-problem so that it is hard for the learner to find a good policy for the $k$-th problem during the first $\sum_{i\leq k-1}T_{i}$ time steps. 

The most interesting sub-problem design is for $k = 2$, which takes the advantage of the context-blind setting to force the learner incur more regret. In this sub-problem, we consider a linear contextual bandit problem $d/M$ arms (assuming that $K\geq d/M$). However, during each time step, only $(d/(2M) + 1)$ arms are presented to the learner. In particular, we set the first $d/(2M)$ arms to be the \emph{frequent} arms that always appear in the context set. For the rest $d/(2M)$ \emph{infrequent} arms, we choose one of them uniformly randomly and include the chosen one in the context set at each time step. The best exploration strategy for the learner is to play the infrequent arm with a higher probability, e.g., $\frac{1}{2}$ and play the frequent arms with probability $\frac{M}{d}$. However, due to the context-blind setting, the learner can not tell whether an arm is frequent or infrequent with insufficient information, and his best strategy is to play the arms with the same probability. As a result, the infrequent arms are insufficiently explored, which leads to an extra $\tilde{\Theta}(\sqrt{d})$ factor in the regret.



\paragraph{Context-aware Lower Bound.} The construction of the lower bound instances for the context-aware setting is quite straightforward. We simply re-define the sub-problem for $k=2$ to be the hard instance for the $K$-armed linear contextual linear bandit problem (the same as $k \in \{3, 4, 5, \dots, M\}$ in the context-blind case) and re-design the stage lengths $\{T_k\}_{k=1}^M$ to achieve the optimal lower bound in this setting. 

\section{Context-blind Batch Learning: Algorithm and Regret Analysis}\label{sec:alg}

We now present our context-blind batch learning algorithm in Algorithm~\ref{alg:main}. Given $T$ and $d$, we define $\tilde{d} :=d\ln(TKd/\delta)$. In the case $d\leq T\leq d\ln(K)$, the regret bound is exactly $\Theta(T)$, and in the case $d\ln(K)<T \leq \tilde{d}$, the regret lower bound is $\Omega(d\ln(K))$ and the upper bound is at most $O(T)\leq O(d\ln(KTd/\delta))$. Therefore, in the case $d\leq T<\tilde{d}$, the trivial upper bound $O(T)$ is optimal up to logarithmic factors in $T$ and $d$. Below we assume $T\geq \tilde{d}$.

Let the time schedule $\{T_{k}\}_{k=1}^{M}$ such that $\sum_{k=1}^{M}T_{k}\geq T$ to be determined later.
We can then accordingly calculate the end of each batch by
\begin{align}
\mathcal{T}_k= \min\left\{\sum_{\ell=1}^{k} T_k, T\right\} .
\end{align}

\medskip

In preparation for explaining the algorithm, we first introduce a few variables and notations used in the algorithm. During the $k$-th batch, the algorithm learns an estimation, namely $\hat{\btheta}_k$, of the hidden vector $\btheta$, as well as an information matrix $\Lambda_k$ that is used to construct the confidence interval for the estimated rewards based on $\hat{\btheta}_k$. More specifically, given the pair $(\Lambda_k, \hat{\btheta}_k)$, we set the confidence interval for the expected reward of any feature vector $\bx$ to be $[\bx^\top \hat{\btheta}_k \pm \alpha \sqrt{\bx^\top \Lambda_k^{-1} \bx}]$, where $[a \pm b]$ denotes the interval $[a-b, a+b]$ and we set 
\begin{align}
\alpha :=  \sqrt{50\ln(KTd/\delta)}.
\end{align}

Given the pair $(\Lambda_k, \hat{\btheta}_k)$, for any context set $X \subseteq \mathbb{R}^d$, we define the following natural elimination procedure based on the corresponding confidence intervals
\begin{align}
\mathcal{E}(X; (\Lambda_k, \hat{\btheta}_k)) := \left\{ \bx \in X : \bx^\top \hat{\btheta}_k + \alpha \sqrt{\bx^\top \Lambda_k^{-1} \bx} \geq  \boldsymbol{y}^\top \hat{\btheta}_k - \alpha \sqrt{\boldsymbol{y}^\top \Lambda_k^{-1} \boldsymbol{y}}, \forall \boldsymbol{y} \in X \right\} .
\end{align}
In words, $\mathcal{E}(X; (\Lambda_k, \hat{\btheta}_k)) $ returns the set of the survived feature vectors, each of which remains possible to hold the highest expected reward when assuming all confidence intervals based on $(\Lambda_k, \hat{\btheta}_k)$ contains their estimation targets. 

By the end of the $k$-th batch, the algorithm would have learned $k$ pairs $\{(\Lambda_i, \hat{\btheta}_i)\}_{i=1}^k$, and we naturally extend our elimination procedure to $\{(\Lambda_i, \hat{\btheta}_i)\}_{i=1}^k$ as follows.
\begin{align}
\mathcal{E}(X; \{(\Lambda_i, \hat{\btheta}_i)\}_{i=1}^k) := \mathop{\cap}_{i=1}^{k} \mathcal{E}(X; (\Lambda_i, \hat{\btheta}_i)) .
\end{align}

When $X \sim D$ and given $\{(\Lambda_i, \hat{\btheta}_i)\}_{i=1}^k$, we denote the distribution of $\mathcal{E}(X; \{(\Lambda_i, \hat{\btheta}_i)\}_{i=1}^k)$ by $D_{k+1}$.

\begin{algorithm}
\caption{{\tt Context-Blind Batch Learning}} \label{alg:main}
\begin{algorithmic}[1]
\STATE{\textbf{Initialize:} $\lambda\leftarrow 10/T$;  $\Lambda_0 \leftarrow \lambda \mathbf{I};$ $\hat{\btheta}_0\leftarrow \boldsymbol{0}$;}
\FOR{$t=1,2,\ldots, \mathcal{T}_1$}
\STATE{Observe $X_t$;}
\STATE{Play the arm with the feature vector $\by_t\sim \pi^{\mathtt{G}}(X_t)$ and receive the reward $r_t$;\label{line:y}}
\ENDFOR
\STATE{$\Lambda_1 \leftarrow \lambda \mathbf{I}+ \sum_{t=1}^{\mathcal{T}_1/2}\by_t\by_t^{\top}$; $\hat{\btheta}_1\leftarrow \Lambda_1^{-1}\sum_{t=1}^{\mathcal{T}_1/2}r_{t}\by_t$;}
\STATE{$\pi_2\leftarrow \mathtt{ExplorationPolicy}\left(  \{\mathcal{E}(X_t, \{\Lambda_{1},\hat{\theta}_{1} \}) \}_{t=\mathcal{T}_1/2+1}^{\mathcal{T}_1} \right)$}
\FOR{$k=2,\ldots,M$}
\FOR{$t= \mathcal{T}_{k-1}+1, \mathcal{T}_{k-1}+2, \ldots , \mathcal{T}_k$}
\STATE{Observe $X_t$;}
\STATE{$X_t^{(k)}\leftarrow \mathcal{E}(X_t,\{\Lambda_{i},\hat{\btheta}_{i}\}_{i=1}^{k-1} )$; \label{line:alg-main-5}}
\STATE{Play the arm with the feature vector $\boldsymbol{y}_t \sim \pi_{k}(X^{(k)}_t)$ and receive the reward $r_t$;}
\ENDFOR
\STATE{$\Lambda_{k}\leftarrow \lambda \mathbf{I} +  \sum_{t=\mathcal{T}_{k-1}+1}^{\mathcal{T}_{k-1}+T_k/2} \boldsymbol{y}_t \boldsymbol{y}_t^{\top}$; $\hat{\btheta}_k \leftarrow \Lambda_k^{-1}  \sum_{t=\mathcal{T}_{k-1}+1}^{\mathcal{T}_{k-1}+T_k/2} r_t \boldsymbol{y}_t$; \label{line:alg-main-8}}
\STATE{$X^{(K+1)} \leftarrow \mathcal{E}(X_t,\{\Lambda_{i},\hat{\btheta}_{i}\}_{i=1}^{k} )\}_{t=\mathcal{T}_{k-1}+T_k/2+1}^{\mathcal{T}_k}$;}
\STATE{$\pi_{k+1}\leftarrow \mathtt{ExplorationPolicy}( X^{(k+1)} )$; \label{line:alg-main-9}}
\ENDFOR
\end{algorithmic}
\end{algorithm}

\medskip

We now explain the key steps of the algorithm. For the first batch, we take actions according to the local optimal design policy $\pi^{\mathtt{G}}$, which is defined by the lemma below.

\begin{lemma}[General Equivalence Theorem in \citep{kiefer1960equivalence}]\label{lemma_od}
For any bounded subset $X\subset \mathbb{R}^d$, there exists a distribution $\mathcal{K}(X)$ supported on $X$, such that for any $\epsilon >0$, it holds that
\begin{align}
\max_{\bx\in X}\bx^{\top} \left( \epsilon \mathbf{I} +\mathbb{E}_{\by\sim \mathcal{K}(X)}[\by\by^{\top}] \right)^{-1}\bx \leq d.\label{eq_od}
\end{align}	
	Furthermore, there exists a mapping $\pi^{\mathtt{G}}$, which maps a context $X$ to a distribution over $X$ such that 
	\begin{align}
    \max_{\bx\in X}\bx^{\top}(\epsilon\mathbf{I}+\mathbb{E}_{\by\sim \pi^{\mathtt{G}}(X) }[\by\by^{\top}])^{-1}\bx \leq 2d.
\end{align}
In particular, when $X$ has a finite size, $\pi^{\mathtt{G}}(X)$ could be implemented within $\mathrm{poly}(|X|)$ time.
\end{lemma}
Clearly, the computation cost to find $\pi^{\mathtt{G}}(X_t)$ is bounded by $\mathrm{poly}(|X_t|)=\mathrm{poly}(K)$. 

For $k\geq 2$, at any time $t$ during the $k$-th batch, the algorithm observes the context set $X_t \sim D$, eliminates some of the sub-optimal arms in Line~\ref{line:alg-main-5}, and denote the set of the survived arms by $X_t^{(k)}$. By the definition above, the $X_t^{(k)}$ follows the distribution $D_{k}$ when conditioned on the first $(k-1)$ batches. The algorithm then uses an exploration policy $\pi_{k-1}$ to stochastically select and play an arm $\boldsymbol{y}_t \in X_t^{(k)}$.

At the end of the $k$-th batch, our algorithm divides the $T_k$ data points collected in the batch into two parts of the equal sizes. In Line~\ref{line:alg-main-8}, our algorithm calculates $\Lambda_k$ and $\hat{\btheta}_k$ using the standard ridge regression and the first part of the data points. In Line~\ref{line:alg-main-9}, the new exploration policy $\pi_{k}$ is learned by the $\mathtt{ExplorationPolicy}$ procedure using the context sets from the second part of the data points. Note that the context sets fed into  $\mathtt{ExplorationPolicy}$ go through the elimination procedure based on $\{\Lambda_i, \hat{\btheta}\}_{i=1}^k$, which depends on the first part of the data points. We will introduce our key procedure $\mathtt{ExplorationPolicy}$ in Section~\ref{sec:exppolicy}. For now, we treat it as a black box and prove our main theorem as follows.

\subsection{The Regret Analysis: Proof of Theorem~\ref{thm:main}}\label{sec:proofthm1}

We first define the following desired event where all the confidence intervals contains their estimation targets,
\begin{align}
E := \left\{\bx^\top \btheta \in [\bx^\top \hat{\btheta}_k \pm \alpha \sqrt{\bx^\top \Lambda_k^{-1} \bx}], \forall k \in \{1, 2, \dots, M\}, \forall \bx \in X_t, \forall t \in \{1, 2, \dots, T\}\right\} .
\end{align}

By the analysis of the ridge regression (Lemma~\ref{lemma:bandit_ci}, stated and proved in Appendix~\ref{sec:pfci}) and the fact that $\alpha \geq \sqrt{\ln(KTd/\delta)}+\lambda \sqrt{d}$, via a union bound we have that $\Pr[E] \geq 1 - MT\delta$. When $E$ holds, we know that the optimal arm at any time step will never be eliminated by the elimination procedure $\mathcal{E}$. Let $\mathbb{I}[E]$ be the indicator variable which takes value $1$ when $E$ holds and value $0$ otherwise. We have the following upper bound for the expected regret of the algorithm.
\begin{align}
R_T& \leq \sum_{k=1}^{M} \sum_{t=\mathcal{T}_{k-1}+1}^{\mathcal{T}_k} \E\left[\mathbb{I}[E]  \left( \max_{\bx \in X_t} \{\bx^\top \btheta\} - \boldsymbol{y}_t^\top \btheta \right)\right] + 2T \Pr[\overline{E}] \nonumber \\
& \leq \sum_{k=1}^{M}\sum_{t=\mathcal{T}_{k-1}+1}^{\mathcal{T}_k} \E\left[\mathbb{I}[E]   \left( \max_{\bx \in X_t^{(k)}} \{\bx^\top \btheta\} - \min_{\bx \in X_t^{(k)}} \{\bx^\top \btheta\}  \right)\right] + 2 MT^2 \delta . \label{eq:ra-1}
\end{align}

\begin{lemma}\label{lemma:r1}
For any time step $t$ during any batch $k$ ($k \geq 2$), define $X_t^{(k-1)} := \mathcal{E}(X_t,\{\Lambda_{i},\hat{\btheta}_{i}\}_{i=1}^{k-2} )$ and we have that
\begin{align}
 \mathbb{E}\left[\mathbb{I}[E] \left(\max_{\bx\in X^{(k)}_{t} } \{\bx^{\top}\btheta\} - \min_{\bx\in X_t^{(k)}} \{\bx^{\top}\btheta \}\right) \right]\leq  \mathbb{E}\left[ \min\left\{  4\alpha\max_{\bx\in X_t^{(k)} }\sqrt{\bx^{\top}\Lambda_{k-1}^{-1}\bx}, 2\right\} \right].\label{eq_reg1}
\end{align}
\end{lemma}
\begin{proof}
When $E$ happens, for any $\bu, \bv \in X_t^{(k)}$, we have $\bu^\top \btheta \in [\bu^\top \hat{\btheta}_{k-1} \pm \alpha \sqrt{\bu^\top \Lambda_{k-1}^{-1} \bu}]$ and $\bv^\top \btheta \in [\bv^\top \hat{\btheta}_{k-1} \pm \alpha \sqrt{\bv^\top \Lambda_{k-1}^{-1} \bv}]$. When this condition holds, we have that
\begin{align}
    & \mathbb{I}[E] \left(\bu^{\top}\btheta - \bv^{\top}\btheta\right) \nonumber  \\
    & \leq  \left(\bu^{\top}\hat{\btheta}_{k-1} + \alpha\sqrt{\bu^{\top}\Lambda_{k-1}^{-1}\bu} \right) - \left(\bv^{\top}\hat{\btheta}_{k-1} - \alpha\sqrt{\bv^{\top}\Lambda_{k-1}^{-1}\bv} \right) \nonumber
    \\ & = \left(\bu^{\top}\hat{\btheta}_{k-1} -\alpha\sqrt{\bu^{\top}\Lambda_{k-1}^{-1}\bu} \right)-\left(\bv^{\top}\hat{\btheta}_{k-1} + \alpha\sqrt{\bv^{\top}\Lambda_{k-1}^{-1}\bv} \right) + 2\alpha\sqrt{\bu^{\top}\Lambda_{k-1}^{-1}\bu} + 2\alpha\sqrt{\bv^{\top}\Lambda_{k-1}^{-1}\bv}\nonumber
    \\ & \leq 2\alpha\sqrt{\bu^{\top}\Lambda_{k-1}^{-1}\bu} + 2\alpha\sqrt{\bv^{\top}\Lambda_{k-1}^{-1}\bv} \label{eq:lem-r1-1}
    \\ & \leq 4\alpha\max_{\bx\in X_{t}^{(k)}}\sqrt{\bx^{\top} \Lambda_{k-1}^{-1}\bx}  \leq  4\alpha\max_{\bx\in X_{t}^{(k-1)}}\sqrt{\bx^{\top} \Lambda_{k-1}^{-1}\bx}, \nonumber
\end{align}
where the \eqref{eq:lem-r1-1} is because $\bv$ survived from the elimination based on $(\Lambda_{k-1}, \hat{\btheta}_{k-1})$, and the last inequality is because $X_{t}^{(k)}\subseteq X_{t}^{(k-1)}$. 

Letting $\bu = \arg\max_{\bx \in X_t^{(k)}} \{\bx^\top \btheta\}$ and $\bv = \arg\min_{\bx \in X_t^{(k)}} \{\bx^\top \btheta\}$ and noting that $\bu^\top \btheta, \bv^\top\btheta \in [-1, 1]$, we have that
\begin{align}
\mathbb{I}[E] \left(\max_{\bx\in X^{(k)}_{t} } \{\bx^{\top}\btheta\} - \min_{\bx\in X_t^{(k)}} \{\bx^{\top}\btheta \}\right) \leq \min\left\{4\alpha\max_{\bx\in X_{t}^{(k-1)}}\sqrt{\bx^{\top} \Lambda_{k-1}^{-1}\bx}, 2\right\}. \label{eq:lem-r1-2}
\end{align}
Taking the expectation over \eqref{eq:lem-r1-2}, we prove the lemma.
\end{proof}

\subsubsection{Regret in the First and Second Batches}

The regret in the first batch is bounded by $T_1$ trivially. For the second batch, 
we have the lemma below.
\begin{lemma}\label{lemma:b2}
With probability $1-\delta$, it holds that 
\[\mathbb{E}_{X\sim D}\left[\max_{\bx\in \mathcal{E}(X; (\Lambda_1,\hat{\btheta}_1)) }\min\left\{\sqrt{\bx^{\top}\Lambda_{1}^{-1}\bx} ,2\right\}\right] \leq O\left(\sqrt{\frac{\min\{K,d\}}{\mathcal{T}_1}\cdot\left(d\ln\left(\frac{\mathcal{T}_1}{\lambda}\right)+\ln(\mathcal{T}_1/\delta)\right)  }  \right).\]
\end{lemma}
\begin{proof}
Recall the definition of $\by_t$ and $r_t$ in line~\ref{line:y}, Algorithm~\ref{alg:main}. Let $W_{t}$ denote $\lambda \mathbf{I}+\sum_{\tau=1}^t \by_{\tau}\by_{\tau}^{\top}$ for $0\leq t\leq \mathcal{T}_1/2$. By the Elliptical Potential Lemma (Lemma~\ref{lemma:epl}, stated and proved in Appendix~\ref{app:pf-epl}), we have that
\begin{align}
    \sum_{t=1}^{\mathcal{T}_1/2} \min\{\by_{t}^{\top}W_{t-1}^{-1}\by_{t},1 \}\leq O\left(d\ln\left(\frac{\mathcal{T}_1}{\lambda}\right)\right).\nonumber
\end{align}

By Corollary~\ref{coro1} with $\epsilon=\frac{1}{2}$, we further have that with probability $1-\delta$, it holds that
\begin{align}
    \sum_{t=1}^{\mathcal{T}_1/2} \mathbb{E}_{X\sim D}\mathbb{E}_{\by\sim \pi^{\mathtt{G}}(X)} \min\{ \by^{\top}W_{t-1}^{-1}\by,1\} & \leq 2 \sum_{t=1}^{\mathcal{T}_1/2} \min\{\by_{t}^{\top}W_{t-1}^{-1}\by_{t},1 \}+ 56\ln(\mathcal{T}_1/\delta)\nonumber
 \\ &\leq O\left(d\ln\left(\frac{\mathcal{T}_1}{\lambda}\right)+\ln(\mathcal{T}_1/\delta)\right).\label{eq:r112}
\end{align}

Noting that $W_{t-1}\preccurlyeq W_{t}$ for $1\leq t\leq \mathcal{T}_1$ and $\Lambda_{1}=W_{\mathcal{T}_1/2}$, by \eqref{eq:r112} we have that
\begin{align}
   \mathcal{T}_1\mathbb{E}_{X\sim D}\mathbb{E}_{\by\sim \pi^{\mathtt{G}}(X)} \min\{ \by^{\top}\Lambda_{1}^{-1}\by,1\}\leq O\left(d\ln\left(\frac{\mathcal{T}_1}{\lambda}\right)+\ln(\mathcal{T}_1/\delta)\right).\label{eq:r113}
\end{align}

\begin{lemma}\label{lemma:addd1}
For any PSD matrix $w$ and context $X$, we have that
\begin{align}
\max_{\bx\in X} \{\bx^{\top}W^{-1}\bx \} \leq \min\{K,d \}   \mathbb{E}_{\by\sim \pi^{\mathtt{G}}(X)}  \by^{\top}W^{-1}\by.\label{eq:r1141}
\end{align}
\end{lemma}

\begin{proof}

for any $x$, the confidence region is $\|xW^{-1/2}\|_1$. Also we have that $\|x\|_{W(x),2}\leq \sqrt{d}$. $W = \mathbb{E}[X]$ and we need to bound $\mathbb{E}[ \mathrm{Tr}(\sqrt{XW^{-1})} \sqrt{d}         ]\leq d$. That is enough.

\begin{lemma}
For any $x$ such that $x^{\top}X^{-1}x\leq 1$, we have that $\|xW^{-1/2}\|^2_1 \leq \mathrm{Tr}(XW^{-1})$. 
\end{lemma}
\begin{proof}
 $\|xW^{-1/2}\|^2_1 \sim \|y X^{1/2}W^{-1/2}    \|^2_1\leq   (|y_1v_1|+ |y_2v_2|+...|y_dv_d|)^2 \leq \sqrt{ \sum_{i=1}^d \|v_i\|^2_{2}      }$. Use the fact the $X$ and $W$ are PSD matrices.. Consider to use the engi-space of $X$ instead of $W$.
\end{proof}


By Lemma~\ref{lemma_od}, we have that
\begin{align}
\max_{\bx\in X} \{ \bx^{\top}W^{-1}\bx \} \leq \min\{K,d \}   \mathbb{E}_{\by\sim \pi^{\mathtt{G}}(X)}  \by^{\top}W^{-1}\by.\label{eq:r114}
\end{align}
In the case $\max_{\bx\in X}\bx^{\top}W^{-1}\bx\leq 1$, we have that $\by^{
\top}W^{-1}
\by\leq 1$ for any $y\in X$. It then holds that
\begin{align}
     &  \max_{\bx\in X}\min\{\bx^{\top}W^{-1}\bx,1\}\leq \min\{K,d\}\mathbb{E}_{\by\sim \pi^{\mathtt{G}}(X)} \min\{ \by^{\top}W^{-1}\by,1 \}.\nonumber
\end{align}
In the case $\max_{\bx\in X}\bx^{\top}W^{-1}\bx>1$, we analyze as below.
When $\min\{K,d \}\mathrm{Pr}_{\by\sim \pi^{\mathtt{G}}(X)}[\by^{\top}W^{-1}\by>1]\geq 1$, it is trivial that
\begin{align}
   \min\{K,d\}\mathbb{E}_{\by\sim \pi^{\mathtt{G}}(X)} \min\{ \by^{\top}W^{-1}\by,1 \}\geq 1.\nonumber
\end{align}
Otherwise, we have that
\begin{align}
  &  \min\{K,d\}\mathbb{E}_{\by\sim \pi^{\mathtt{G}}(X)} \min\{ \by^{\top}W^{-1}\by,1 \}\nonumber
    \\ &\geq \min\{K,d\}\mathbb{E}_{\by\sim \pi^{\mathtt{G}}(X)} \by^{\top}W^{-1}\by -\min\{K,d \}\mathrm{Pr}_{\by\sim \pi^{\mathtt{G}}(X)}[\by^{\top}W^{-1}\by>1] (\max_{\bx\in X}\bx^{\top}W^{-1}\bx-1)\nonumber
    \\ & \geq  \max_{\bx\in X}\bx^{\top}W^{-1}\bx-\min\{K,d \}\mathrm{Pr}_{\by\sim \pi^{\mathtt{G}}(X)}[\by^{\top}W^{-1}\by>1] (\max_{\bx\in X}\bx^{\top}W^{-1}\bx-1)\label{eq:addd2}
    \\ & = 1+ (1-\min\{K,d \}\mathrm{Pr}_{\by\sim \pi^{\mathtt{G}}(X)}[\by^{\top}W^{-1}\by>1])( \max_{\bx\in X}\bx^{\top}W^{-1}\bx-1)\nonumber
    \\ & \geq 1.\nonumber
\end{align}
Here \eqref{eq:addd2} holds by Lemma~\ref{lemma_od}.
The lemma is proved.
\end{proof}

Setting $W=\Lambda_1$ in Lemma~\ref{lemma:addd1}, we have that
\begin{align}
    \mathbb{E}_{X\sim D}\max_{\bx\in X}\min\{\bx^{\top}\Lambda_{1}^{-1}\bx,1\} \leq \frac{\min\{K,d \}}{\mathcal{T}_1}\cdot O\left(d\ln\left(\frac{\mathcal{T}_1}{\lambda}\right)+\ln(\mathcal{T}_1/\delta)\right).\label{eq:r117}
\end{align}

Therefore, 
\begin{align}
&\mathbb{E}_{X\sim D}\left[\max_{\bx\in \mathcal{E}(X; (\Lambda_1,\hat{\btheta}_1)) }\sqrt{\min\{\bx^{\top}\Lambda_{1}^{-1}\bx,1\}}  \right]\nonumber
\\ & \qquad \leq  \sqrt{\mathbb{E}_{X\sim D}\left[\max_{\bx\in X }\min\{ \bx^{\top}\Lambda_1^{-1}\bx ,1\} \right] } \leq O\left(\sqrt{\frac{\min\{K,d\}}{\mathcal{T}_1}\cdot\left(d\ln\left(\frac{\mathcal{T}_1}{\lambda}\right)+\ln(\mathcal{T}_1/\delta)\right)  }  \right).\nonumber
\end{align}
Lemma~\ref{lemma:b2} is proved.
\end{proof}

By Lemma~\ref{lemma:r1} and~\ref{lemma:b2}, the regret in the second  batch is bounded by 
\begin{align}
    O\left(T_{2}\cdot\sqrt{\frac{\min\{K,d\}d}{T_1}\ln\left(\frac{T}{\lambda\delta} \right)}  \right) +T\delta.
\end{align}

\subsubsection{Regret in the $k$-th Batch ($k\geq 3$)}

Let 
\begin{align}
L := \frac{1}{200 \ln (Td/\delta)} \qquad \text{and}\qquad \kappa := \frac{1}{T^2}
\end{align}
be the parameters to be used by $\mathtt{ExplorationPolicy}$.

In Section~\ref{sec:exppolicy}, we will introduce $\mathtt{ExplorationPolicy}$ and prove the following lemma.

\begin{lemma}\label{lemma:design}
Let  $\{Z_u\}_{u=1}^m$ be $m$ {\it i.i.d.}~stochastic context sets following a distribution $D$. Let $\pi$ be the output by running $\mathtt{ExplorationPolicy}$ with the input $\{Z_u\}_{u=1}^m$. Let $\{\tilde{Z}_{u} \}_{u=1}^n$ be another group of {\it i.i.d.}~stochastic context sets following the distribution $D$ (which is also independent from $\{ Z_u\}_{u=1}^m$). Let $\boldsymbol{y}_u$ be independently sampled from $\pi(\tilde{Z}_u)$ for each $u \in \{1, 2, \dots, n\}$, and let $\Lambda = \sum_{u=1}^n \boldsymbol{y}_u \boldsymbol{y}_u^\top$. With probability $(1-3\delta)$, we have that
\begin{align}
\E_{X\sim D}\left[ \min\left\{ \sqrt{\max_{\bx\in X}\bx^{\top}(\Lambda+\frac{n}{m}\kappa \mathbf{I})^{-1}   \bx} , \sqrt{L} \right\}         \right]  \leq \sqrt{\frac{1}{n}}\cdot O\left(\sqrt{d\ln\left(\frac{md}{\kappa}\right)}\right) +\frac{\sqrt{L}}{m}\cdot O\left(d\ln\left(\frac{md}{\kappa}\right)\right) .  \label{eq:lemma-design}
\end{align}	 
\end{lemma}


For the $k$-th batch ($k \geq 3$), we invoke Lemma~\ref{lemma:design} with $m = T_{k-2}/2$, $\{Z_u\}_{u=1}^m = \{X_t^{(k-1)}\}_{t=\mathcal{T}_{k-3}+T_{k-2}/2+1}^{\mathcal{T}_{k-2}}$, $n = T_{k-1}/2$, $\{\tilde{Z}_u\}_{u=1}^n = \{X_t^{(k-1)}\}_{t=\mathcal{T}_{k-2}+1}^{\mathcal{T}_{k-2}+T_{k-1}/2}$, and $D = D_{k-1}$, we have that with probability $(1-3 \delta)$, it holds that
\begin{align}
&\E_{X\sim D_{k-1}}\left[ \min\left\{ \sqrt{\max_{\bx\in X}\bx^{\top}\left(\sum_{t=\mathcal{T}_{k-2}+1}^{\mathcal{T}_{k-2} + T_{k-1}/2} \boldsymbol{y}_t \boldsymbol{y}_t^\top +\frac{T_{k-1}}{T_{k-2}}\kappa \mathbf{I}\right)^{-1}   \bx} , \sqrt{L} \right\}         \right] \nonumber\\
& \qquad\qquad\qquad \leq \sqrt{\frac{1}{T_{k-1}}}\cdot O\left(\sqrt{d\ln\left(\frac{T_{k-2}d}{\kappa}\right)}\right) +\frac{\sqrt{L}}{T_{k-2}}\cdot O\left(d\ln\left(\frac{T_{k-2}d}{\kappa}\right)\right) . \label{eq:ra-10}
\end{align}
Note that $\frac{T_{k-1}}{T_{k-2}} \kappa \leq \lambda$ and therefore $(\sum_{t=\mathcal{T}_{k-2}+1}^{\mathcal{T}_{k-2} + T_{k-1}/2}  \boldsymbol{y}_t \boldsymbol{y}_t^\top +\frac{T_{k-1}}{T_{k-2}}\kappa \mathbf{I})^{-1} \succcurlyeq \Lambda_{k-1}^{-1}$. Therefore \eqref{eq:ra-10} implies that with probability $(1-3\delta)$,
\begin{align}
&\E_{X\sim D_{k-1}}\left[ \min\left\{ \sqrt{\max_{\bx\in X}\bx^{\top}\Lambda_{k-1}^{-1}   \bx} , \sqrt{L} \right\}         \right] \leq \sqrt{\frac{1}{T_{k-1}}}\cdot O\left(\sqrt{d\ln\left(Td\right)}\right) +\frac{\sqrt{L}}{T_{k-2}}\cdot O\left(d\ln\left(Td\right)\right) . 
\end{align}
In other words, for each time step $t$ during the $k$-th batch, we have that
\begin{align}
\E\left[\min\left\{ \sqrt{\max_{\bx\in X_t^{(k-1)}}\bx^{\top}\Lambda_{k-1}^{-1}   \bx} , \sqrt{L} \right\}    \right] \leq \sqrt{\frac{1}{T_{k-1}}}\cdot O\left(\sqrt{d\ln\left(Td\right)}\right) +\frac{\sqrt{L}}{T_{k-2}}\cdot O\left(d\ln\left(Td\right)\right)  + 3\delta .  \label{eq:ra-12}
\end{align}

\subsubsection{Putting All Together}
Combining \eqref{eq:ra-1}, Lemma~\ref{lemma:r1} and Lemma~\ref{lemma:b2}, the total regret is bounded by 
\begin{align}
R_{T}&\leq 2T_1+2T_2\cdot \min\left\{\alpha \sqrt{\frac{2\min\{K,d\}d\ln\left(\frac{T}{\lambda\delta}\right)}{T_1}} ,1\right\}\nonumber\\
&\qquad \qquad \qquad \qquad  \qquad \qquad + \sum_{k=3}^M T_k  \mathbb{E}\left[ \min\left\{  4\alpha\max_{\bx\in X_t^{(k-1)} }\sqrt{\bx^{\top}\Lambda_{k-1}^{-1}\bx}, 2\right\} \right] + 2MT^2\delta \nonumber \\
&\leq  2T_1+2T_2\cdot \min\left\{\alpha \sqrt{\frac{2\min\{K,d\}d\ln\left(\frac{T}{\lambda\delta}\right)}{T_1}} ,1\right\} \nonumber \\
&\qquad \qquad  \qquad \qquad  \qquad \qquad + \sum_{k=3}^M 4\alpha T_k  \mathbb{E}\left[ \min\left\{  \max_{\bx\in X_t^{(k-1)} }\sqrt{\bx^{\top}\Lambda_{k-1}^{-1}\bx}, \sqrt{L}\right\} \right] + 2MT^2\delta \label{eq:ra-8.5}\\
&\leq 2T_1+2T_2\cdot \min\left\{\alpha \sqrt{\frac{2\min\{K,d\}d\ln\left(\frac{T}{\lambda\delta}\right)}{T_1}} ,1\right\} \nonumber \\
&\qquad \qquad\qquad \qquad\qquad \qquad + \sum_{k=3}^M 4\alpha T_k \mathbb{E}\left[ \min\left\{  \max_{\bx\in X_t^{(k-1)} }\sqrt{\bx^{\top}\Lambda_{k-1}^{-1}\bx}, \sqrt{L}\right\} \right] + 2MT^2\delta , \label{eq:ra-9}
\end{align}
Here \eqref{eq:ra-8.5} is because of $2\leq 4\alpha \sqrt{L}$.

Let $\tilde{d} :=d\ln(TdK/\delta)\cdot \ln\left(\frac{T}{\lambda\delta}\right)$.
Combining \eqref{eq:ra-9}  and \eqref{eq:ra-12} (for $k \geq 3$), we have that
\begin{align}
&R_T  \leq 2 T_1 +2T_2\cdot \min\left\{ \alpha\sqrt{\frac{2\min\{K,d\}d\ln\left(\frac{T}{\lambda\delta}\right)}{T_1}} ,1\right\}   \nonumber
\\ & \qquad \qquad  \qquad + 4\alpha \sum_{k=3}^{M} \left({\frac{T_k}{\sqrt{T_{k-1}}}}\cdot O\left(\sqrt{d\ln\left(Td\right)}\right) +\frac{T_k \sqrt{L}}{T_{k-2}}\cdot O\left(d\ln\left(Td\right)\right) \right) + 3 T \delta + 2MT^2 \delta \nonumber \\
&  \quad \leq 2T_1 + 2T_2\cdot \min\left\{ \alpha\sqrt{\frac{2\min\{K,d\}d\ln\left(\frac{T}{\lambda\delta}\right)}{T_1}} ,1\right\}  + O(\sqrt{\ln^2 (Td/\delta)+\ln(K)\ln(Td/\delta) }) \times \sum_{k=3}^{M} \frac{T_k \sqrt{d}}{\sqrt{T_{k-1}}} \nonumber
\\ & \qquad  \qquad \qquad       + O\left(\sqrt{\ln(K)\ln^2(Td)+\ln^3(Td)} \right)\times \sum_{k=3}^M\frac{T_k d}{T_{k-2}}  + O(MT^2 \delta) \nonumber
\\ & \quad \leq 2T_1 +  2T_2\cdot \min\left\{ \sqrt{\frac{2\min\{K,d\}\tilde{d}}{T_1}} ,1\right\}  + O(\ln(Td)) \cdot\left( \sum_{k=3}^{M} \frac{T_k\sqrt{\tilde{d}} }{\sqrt{T_{k-1}}}+ \sum_{k=3}^M\frac{T_k\tilde{d} }{T_{k-2}} \right) + O(MT^2 \delta).\label{eq:c30}
\end{align}

When $T$ is small (i.e., $d\leq T<\tilde{d}$), Theorem~\ref{thm:main} trivially holds because the regret is at most $O(\tilde{d})$, which is further bounded by 
\[ O\left(\mathrm{poly}\ln(Td)\min\left\{ T^{\frac{1}{2-2^{-M+2}}}(d\ln(K))^{\frac{1-2^{-M+2}}{2-2^{-M+2}}},T^{\frac{1}{2-2^{-M+1}}}(d\ln(K))^{\frac{1-2^{-M+1}}{2-2^{-M+1}}}\min\{K,d\}^{\frac{2^{-M+1}}{2-2^{-M+1}}}\right\}   \right).\]

Suppose $T\geq \tilde{d}$. Let $h =\min\{K,d\}$.
We discuss the following two cases to design the batch schedule.

\paragraph{Case \uppercase\expandafter{\romannumeral1}: $\tilde{d} \leq T\leq \tilde{d}h^{2-2^{-M+2}}$.}
In this case , we define $\gamma :=  T^{\frac{1}{2-2^{-M+2}}}\tilde{d}^{\frac{1-2^{-M+2}}{2-2^{-M+2}}}\geq \tilde{d}$. 

We let 

\begin{align}T_1=\gamma, \quad T_2  =\gamma, \quad T_k = \gamma \frac{\sqrt{T_{k-1}}}{\sqrt{\tilde{d}}}, \forall 3\leq k \leq M.
\label{eq:defineT}\end{align}
Then for $3\leq k\leq M$, by the iteration rule we have that $T_k = \gamma^{2-2^{-k+2}}\tilde{d}^{-1+2^{-k+2}}$. It is easy to verify that $\sum_{k=1}^{M}T_k \geq T_{M}=T$. 

Now we verify that the regret for each batch is bounded by $O(\gamma)$.
Firstly we have that $T_1,T_2\leq \gamma$. For $k=3$, we have that $\frac{T_3\sqrt{d}}{\sqrt{T_{2}}}= \gamma$ and $\frac{T_3 \tilde{d}}{T_1}=\gamma^{\frac{1}{2}}\tilde{d}^{\frac{1}{2}}\leq \gamma$.
For $4\leq k\leq M$, noting that $\tilde{d}\leq \gamma$, we have that
\begin{align}
\frac{T_k\sqrt{d}}{\sqrt{T_{k-1}}}= \gamma ,\quad \quad \quad \frac{T_kd}{T_{k-2}}= \gamma^{2^{-k}-2^{-k+2}}\tilde{d}^{1-2^{-k}+2^{-k+2}}  \leq \gamma.\label{eq:defineT_1}
\end{align}

Therefore, the total regret in this case is bounded by 
\begin{align}
O\left(T^{\frac{1}{2-2^{-M+2}}}\tilde{d}^{\frac{1-2^{-M+2}}{2-2^{-M+2}}}\log\log(T)+MT^2\delta\right).\label{eq:bound1}
\end{align}

\paragraph{Case \uppercase\expandafter{\romannumeral2}:  $T>\tilde{d}h^{2-2^{-M+2}}$.} In this case, we define $\gamma:= T^{\frac{1}{2-2^{-M+1}}}\tilde{d}^{\frac{1-2^{-M+1}}{2-2^{-M+1}}}h^{\frac{2^{-M+1}}{2-2^{-M+1}}}\geq \tilde{d} $. Let 
\begin{align}
T_1=\gamma,\quad T_2= \gamma \frac{\sqrt{T_1}}{\sqrt{\tilde{d}h}},\quad T_k = \gamma \frac{\sqrt{T_{k-1}}}{\sqrt{\tilde{d}}}, \forall 3\leq k\leq M.\label{eq:definetcase2}
\end{align}By the iteration rule, we have that $T_k = \gamma^{2-2^{-k+1}}\tilde{d}^{-1+2^{-k+1}}h^{-2^{-k+1}}$. In particular,  
$\sum_{k=1}^{M}T_k \geq T_M =\gamma^{2-2^{-M+1}}\tilde{d}^{-1+2^{-M+1}}h^{-2^{-M+1}}=T $.

By definition, we have that $T_1\leq \gamma$ and $T_2 \frac{\sqrt{\tilde{d}h}}{\sqrt{T_1}}\leq \gamma$. For $k=3$, we have that
\begin{align}
     \frac{T_3\sqrt{\tilde{d}}}{\sqrt{T_{2}}}= \gamma,\quad \quad \quad    \frac{T_3\tilde{d}}{T_1}=\gamma^{\frac{3}{4}}\tilde{d}^{\frac{1}{4}}h^{-\frac{1}{4}}\leq \gamma.\nonumber
\end{align}

For $4\leq k \leq M$, we have that
\begin{align}
 \frac{T_k\sqrt{\tilde{d}}}{\sqrt{T_{k-1}}}= \gamma,\quad \quad \quad    \frac{T_k\tilde{d}}{T_{k-2}}=\gamma^{2^{-k+3}-2^{-k+1}}\tilde{d}^{1+2^{-k+1}-2^{-k+3}}h^{2^{-k+3}-2^{-k+1}}\leq \tilde{d}\left(\frac{\gamma h}{\tilde{d}}\right)^{\frac{3}{8}}.
\end{align}
So it suffices to verify $\tilde{d}h^{\frac{3}{5}}\leq \gamma$. In fact we have that
\begin{align}
    \gamma^{2-2^{-M+1}}=T\tilde{d}^{1-2^{-M+1}}h^{2^{-M+1}}\geq \tilde{d}^{2-2^{-M+1}}h^{2-2^{-M+2}+2^{-M+1}} \geq \tilde{d}^{2-2^{-M+1}}h^{\frac{3}{5}(2-2^{-M+1})},
\end{align}
which implies that $\tilde{d}h^{\frac{3}{5}}\leq \gamma$.

Therefore, the total regret in this case is bounded by 
\begin{align}
O\left(T^{\frac{1}{2-2^{-M+1}}}\tilde{d}^{\frac{1-2^{-M+1}}{2-2^{-M+1}}}h^{\frac{2^{-M+1}}{2-2^{-M+1}}}\log\log(T)+MT^2\delta\right) \label{eq:bound2}.
\end{align}

We now finish the discuss about the two cases and combine the two regret upper bounds \eqref{eq:bound1} and \eqref{eq:bound2}. Noting that $T>\tilde{d}h^{2-2^{-M+2}}$ implies that
\begin{align}
  T^{\frac{1}{2-2^{-M+2}}}\tilde{d}^{\frac{1-2^{-M+2}}{2-2^{-M+2}}}> T^{\frac{1}{2-2^{-M+1}}}\tilde{d}^{\frac{1-2^{-M+1}}{2-2^{-M+1}}}h^{\frac{2^{-M+1}}{2-2^{-M+1}}},
\end{align}
we have that
\begin{align}
   R_T & \leq O\left(\min\left\{ T^{\frac{1}{2-2^{-M+2}}}\tilde{d}^{\frac{1-2^{-M+2}}{2-2^{-M+2}}},T^{\frac{1}{2-2^{-M+1}}}\tilde{d}^{\frac{1-2^{-M+1}}{2-2^{-M+1}}}h^{\frac{2^{-M+1}}{2-2^{-M+1}}}\right\}\cdot \log\log(T)+MT^2\delta  \right).\nonumber
\end{align}

Setting $\delta = 1/T^3$, we obtain that
\begin{align}
R_T & \leq O\Bigg(\mathrm{poly}\ln(Td) \min\Big\{ T^{\frac{1}{2-2^{-M+2}}}(d\ln(K))^{\frac{1-2^{-M+2}}{2-2^{-M+2}}}, \nonumber 
\\ & \qquad \qquad \qquad \qquad \qquad \qquad\qquad\quad  T^{\frac{1}{2-2^{-M+1}}}(d\ln(K))^{\frac{1-2^{-M+1}}{2-2^{-M+1}}}\min\{K,d\}^{\frac{2^{-M+1}}{2-2^{-M+1}}}\Big\} \Bigg).\nonumber
\end{align}
Theorem~\ref{thm:main} is proven. \qed

\section{Learning the Exploration Policy} \label{sec:exppolicy}

In this section, we formally describe the $\mathtt{ExplorationPolicy}$ procedure by Algorithm~\ref{alg:exppolicy}. Suppose there is an unknown distribution $D$ over the context sets. Given a set of $m$ independent samples $\{Z_i\}_{i=1}^m$ drawn from $D$, the goal of Algorithm~\ref{alg:exppolicy} is to learn an exploration policy $\pi$ so that if one uses $\pi$ to collect $n$ more data points (including context vectors and observed rewards) and estimate the linear model $\btheta$, the expected size of the largest confidence interval among all actions in a random context set (as characterized by the LHS of \eqref{eq:lemma-design} in Lemma~\ref{lemma:design}) will be small.

We now briefly explain our Algorithm~\ref{alg:exppolicy}. Given a group of context vectors $\{ Z_i\}_{i=1}^m$, the algorithm simulates the reward-free linear bandit algorithms. In each time step, the algorithm first clip the context vectors according the current information matrix $W$, and then chooses the arm with clipped maximal variance. The information matrix is updated with doubling trick, which helps reduce both the number of updates and the complexity of the output policy.

\begin{algorithm}
\caption{$\mathtt{ExplorationPolicy}$}\label{alg:exppolicy}
\begin{algorithmic}[1]
\STATE{\textbf{Input:}  $\{Z_i\}_{i=1}^m$;}
\STATE{\textbf{Initialization:} $\kappa = \frac{1}{T^2}$ $U_0\leftarrow \kappa \mathbf{I}$; $\eta \leftarrow 1$, $\tau_{\eta}\leftarrow \emptyset, W_{\eta}\leftarrow U_0$;}
\FOR{$i=1,2,\ldots,m$}
    \STATE{$\tau_\eta \leftarrow \tau_\eta \cup \{i\}$;}
    \STATE{Choose $\bz_{i} \in Z_i$ to maximize $\bz_i^{\top}W_{\eta}^{-1}\bz_i$;}
    \STATE{\label{line:alg-exppolicy-6} $\tilde{\bz}_i \leftarrow \min\left\{ \sqrt{\frac{L}{\bz_i^{\top}W_{\eta}^{-1}\bz_i}}, 1\right\}  \bz_i$; $U_i\leftarrow U_{i-1}+\tilde{\bz}_i \tilde{\bz}_i^\top$;}
    \IF{$\det(U_i)>2\det(W_\eta)$}
    \STATE{$\eta\leftarrow \eta+1,\tau_{\eta}\leftarrow \emptyset, W_\eta \leftarrow U_i$;}
    \ENDIF
\ENDFOR
\STATE{Let $\pi$ be the policy such that 
\begin{align}
    \pi(X) = \arg\max_{\bx \in X}\{\bx^\top W_j^{-1} \bx\} \qquad \text{with probability~} \frac{|\tau_j|}{m} \text{~for~}j \in \{1, 2, \dots, \eta\};
\end{align}
}
\STATE{\textbf{return:} $\pi$;}
\end{algorithmic}
\end{algorithm}

In the rest of this section, we will prove Lemma~\ref{lemma:design} on the guarantee of $\mathtt{ExplorationPolicy}$. For the readers' convenience, we re-state the lemma as follows.

\medskip
\noindent {\bf Lemma~\ref{lemma:design} (restated).} {\it
Let  $\{Z_u\}_{u=1}^m$ be $m$ {\it i.i.d.}~stochastic context sets following a distribution $D$. Let $\pi$ be the output by running $\mathtt{ExplorationPolicy}$ with the input $\{Z_u\}_{u=1}^m$. Let $\{\tilde{Z}_{u} \}_{u=1}^n$ be another group of {\it i.i.d.}~stochastic context sets following the distribution $D$ (which is also independent from $\{ Z_u\}_{u=1}^m$). Let $\boldsymbol{y}_u$ be independently sampled from $\pi(\tilde{Z}_u)$ for each $u \in \{1, 2, \dots, n\}$, and let $\Lambda = \sum_{u=1}^n \boldsymbol{y}_u \boldsymbol{y}_u^\top$. With probability $(1-3\delta)$, we have that
\begin{align*}
\E_{X\sim D}\left[ \min\left\{ \sqrt{\max_{\bx\in X}\bx^{\top}(\Lambda+\frac{n}{m}\kappa \mathbf{I})^{-1}   \bx} , \sqrt{L} \right\}         \right]  \leq \sqrt{\frac{1}{n}}\cdot O\left(\sqrt{d\ln\left(\frac{md}{\kappa}\right)}\right) +\frac{\sqrt{L}}{m}\cdot O\left(d\ln\left(\frac{md}{\kappa}\right)\right) . \\
\eqref{eq:lemma-design}
\end{align*}	 
}


\subsection{Analysis: Proof of Lemma~\ref{lemma:design}} \label{sec:proof-lemma-design}
As stated in Section~\ref{sec:tec}, 
the proof of Lemma~\ref{lemma:design} utilizes similar ideas in the proof of Theorem 5 in \citep{ruan2020linear}. The major difference is that their information matrix starts with $\Omega(1)\cdot \mathbf{I}$ when executing the output policy, while our information matrix could start with $\kappa \mathbf{I}$ with $\kappa = \frac{1}{T^2}$. As a result, it is harder for us to recover the information matrix $U_{m}$.

Let $D$ be the distribution defined in the statement of Lemma~\ref{lemma:design}, we first prove the following lemma.

\begin{lemma}\label{claim:emp}
With probability $(1-\delta)$, it holds that
\begin{align}
m \E_{X\sim D} \left[   \min \{   \max_{\bx \in X}\bx^{\top}U_{m}^{-1}\bx   ,L   \}  \right]\leq O(d\ln(md/\kappa)).\label{eq:emp1}
\end{align}
\end{lemma}

\begin{proof}
Let $\eta_i$ denote the index $\eta$ such that $i\in \tau_{\eta}$. Note that $i\in \tau_{\eta}$ implies that $\det(U_{i-1})\leq 2\det(W_{\eta})$, which further implies that $\det((W_{\eta}^{-\frac{1}{2}})^{\top}U_{i-1}(W_{\eta})^{\frac{1}{2}} )\leq 2$. Because $U_{i-1}\succcurlyeq W_{\eta}$, we have that $(W_{\eta}^{-\frac{1}{2}})^{\top}U_{i-1}(W_{\eta})^{\frac{1}{2}}\succcurlyeq \mathbf{I}$. Therefore, the maximal eigenvalue of $(W_{\eta}^{-\frac{1}{2}})^{\top}U_{i-1}(W_{\eta})^{\frac{1}{2}}\succcurlyeq \mathbf{I}$ is at most $2$, where it follows that $U_{i-1}\leq 2W_{\eta}$.

Since $U_m \succcurlyeq W_{\eta}$ for all $\eta\geq 1$, we have that
\begin{align}
m \E_{X\sim D} \left[   \min \{   \max_{\bx \in X}\bx^{\top}U_{m}^{-1}\bx   ,L   \}  \right]  \leq \sum_{i=1}^m  \E_{X\sim D} \left[   \min \{   \max_{\bx \in X}\bx^{\top}W_{\eta_i}^{-1}\bx   ,L   \}  \right]. \label{eq:emp1.5}
\end{align}
Invoking Corollary~\ref{coro1} with $\epsilon = 1$ and noting that $\E[\min\{\tilde{\bz}_i^\top W_{\eta_i}^{-1} \tilde{\bz}_i, L\}] = \E_{X\sim D} \left[   \min \{   \max_{\bx \in X}\bx^{\top}W_{\eta_i}^{-1}\bx   ,L   \}  \right]$ when conditioned on the first $(i - 1)$ iterations in Algorithm~\ref{alg:exppolicy}, we have with probability $(1-\delta)$, it holds that
\begin{align}
\sum_{i=1}^m  \E_{X\sim D} \left[   \min \{   \max_{\bx \in X}\bx^{\top}W_{\eta_i}^{-1}\bx   ,L   \}  \right]    \leq 2\sum_{i=1}^m \min\{  \bz_{i}^{\top}W_{\eta_i}^{-1} \bz_{i}  ,L  \} + 20 L \ln(1/\delta) .  \label{eq:emp2}
\end{align}
By the definition of $\tilde{\bz}_i$, we further have that
\begin{align}
\sum_{i=1}^m \min\{  \bz_{i}^{\top}W_{\eta_i}^{-1} \bz_{i}  ,L  \} = \sum_{i=1}^m   \tilde{\bz}_{i}^{\top}W_{\eta_i}^{-1} \tilde{\bz}_{i}  \leq 2\sum_{i=1}^m   \tilde{\bz}_{i}^{\top}U_{i-1}^{-1} \tilde{\bz}_{i}\leq O(d \ln (md/\kappa)) , \label{eq:emp3}
\end{align}
where the second last inequality is by the fact that  $U_{i-1}\preccurlyeq 2W_{\eta_i}$, and the last inequality is by a direct application of the Elliptical Potential Lemma (Lemma~\ref{lemma:epl}, stated and proved in Appendix~\ref{app:pf-epl}) and the fact that $0<L<1$.
Combining \eqref{eq:emp1.5}, \eqref{eq:emp2},  \eqref{eq:emp3} and the definition of $L$, we prove the lemma.
\end{proof}

 To proceed, we have the lemma below.
\begin{lemma}\label{lemma:emp2} Define $V:= \mathbb{E}_{X\sim D, i\sim \pi(X) }\left[ \min\left\{ \frac{L}{\bx_i^{\top}U_m^{-1}\bx_i },1 \right\} \bx_{i}\bx_{i}^{\top} \right]$.
	With probability $1-\delta$, 
	\begin{align}
V\succcurlyeq \frac{1}{6m}U_m -\frac{1}{3m}\kappa\mathbf{I}.\label{eq:emp6}
	\end{align}
	
\end{lemma}
\begin{proof}

Recall that $\mathcal{F}_{m'} = \sigma(X_1,X_2,\ldots,X_{m'-1})$. 
By the definition of $\pi$, and noting that $U_m \geq W_{\eta_u}$ implies that $\frac{L}{\bx_i^{\top}U_m^{-1}\bx_i}\geq \frac{L}{ \bx_i^{\top}W_{\eta_u}^{-1}\bx_i}$ for $1\leq u \leq m$, we have that
\begin{align}
mV& \succcurlyeq \sum_{u=1}^m \mathbb{E}\left[ \bz_{u,i_u}\bz_{u,i_u}^{\top}|\mathcal{F}_{u} \right] .\nonumber
\end{align}
Noting that  $\bz_{u,i_u}^{\top}W_{\eta_u}^{-1}\bz_{u,i_u}\leq L$ implies $\bz_{u,i_u}\bz_{u,i_u}^{\top}\leq L W_{\eta_u}\leq LU_m$, by Lemma~\ref{lemma:dynamic-martingale-concentration} with $Z_m = LU_m$ and $\epsilon = 2$, with probability $(1-\delta)$, we have that
\begin{align}
mV  \succcurlyeq \mathbb{E}\left[ \bz_{u,i_u}\bz_{u,i_u}^{\top}|\mathcal{F}_{u} \right]  & \succcurlyeq\frac{1}{3}\sum_{u=1}^{m}\bz_{u,i_u}\bz_{u,i_u}^{\top} -\frac{100L\ln(Td/\delta)}{3} U_m \nonumber
\\&= \frac{1}{3}\sum_{u=1}^{m}\bz_{u,i_u}\bz_{u,i_u}^{\top} - \frac{1}{6}U_{m} = \frac{1}{6}U_m - \frac{1}{3}\kappa\mathbf{I},\nonumber
\end{align}
and the conclusion follows by dividing $m$ on both sides of the inequality.
\end{proof}	
	
Recall that $\Lambda = \sum_{u=1}^n \by_{u}\by_{u}^{\top}$. Let 
\begin{align}\tilde{\Lambda}:=\sum_{u=1}^n \min\left\{ \frac{L}{\by_{u}^{\top}U_{m}^{-1}\by_{u}  } ,1 \right\} \by_{u}\by_{u}^{\top}.\nonumber\end{align} 

By definition we have that $\Lambda\succcurlyeq \tilde{\Lambda}$.

Noting that for any $1\leq u \leq n$,
\begin{align}
    \mathbb{E}\left[\min\left\{ \frac{L}{\by_{u}^{\top}U_m^{-1} \by_{u} } ,1 \right\}  \by_{u}\by_{u}^{\top}  \right] =\mathbb{E}_{X\sim D,i\sim \pi(X)}\left[  \min\left\{ \frac{L}{\bx_i^{\top}U_{m}^{-1}\bx_i },1  \right\}\bx_i \bx_i^{\top} \right]=V\nonumber
\end{align}
	and $\min\left\{ \frac{L}{\by_{u}^{\top}U_m^{-1} \by_{u} } ,1 \right\}  \by_{u}\by_{u}^{\top} \preccurlyeq LU_m$, by Corollary~\ref{coro1} with $W=LU_m$ and $\epsilon = \frac{2}{3}$, with probability $1-\delta$, 
	\begin{align}
	    \tilde{\Lambda}\succcurlyeq  \frac{n}{3}V -\frac{68L}{3}U_m.\nonumber
	\end{align}
	By Lemma~\ref{lemma:emp2}, we further have that
	\begin{align}
	    \Lambda &\succcurlyeq \tilde{\Lambda}  \succcurlyeq \frac{n}{3}V -\frac{68L}{3}U_m  \succcurlyeq \frac{n}{3}V-\frac{1}{6}U_m  \succcurlyeq \frac{n}{3m}\left(\frac{1}{6}U_m -\frac{1}{3}\kappa \mathbf{I} \right) -\frac{1}{6}U_m  = \frac{n}{36m}U_m -\frac{n}{9m}\kappa \mathbf{I}.\nonumber
	\end{align}
Therefore, $\Lambda +\frac{n}{m}\kappa \mathbf{I}\geq \frac{n}{36m} U_m $.

In the case $n\geq m$, we have that
\begin{align}
  &    \mathbb{E}_{X\sim D}\left[ \min\left\{ \sqrt{\max_{i}\bx_i^{\top}\left(\Lambda+\frac{n}{m}\kappa \mathbf{I}\right)^{-1}   \bx_i} , \sqrt{L} \right\}         \right] \nonumber
  \\ & \leq \sqrt{\frac{36m}{n}}\mathbb{E}_{X\sim D}\left[ \min\left\{ \sqrt{\max_{i}\bx_i^{\top}U_m^{-1}   \bx_i} , \sqrt{Ln/m} \right\}   \right] \nonumber
  \\ & \leq \sqrt{\frac{36m}{n}} \mathbb{E}_{X\sim D}\left[ \min\left\{ \sqrt{\max_{i}\bx_i^{\top} U_m^{-1}   \bx_i} , \sqrt{L} \right\}   \right]+ \sqrt{\frac{36m}{n}}\cdot \sqrt{\frac{Ln}{m}}\cdot \mathrm{Pr}_{X\sim D}\left[ \max_{i}\bx_i^{\top}U_m^{-1}   \bx_i \geq L \right]\nonumber
  \\ & \leq \sqrt{\frac{36m}{n}}\sqrt{\mathbb{E}_{X\sim D}\left[\min\left\{\max_{i}\bx_i^{\top}U_m^{-1}   \bx_i , L\right\}   \right]  } +\frac{\sqrt{L}}{m}\cdot O\left(d\ln\left(\frac{md}{\kappa}\right)\right)\nonumber
  \\ & \leq \sqrt{\frac{1}{n}}\cdot O\left(\sqrt{d\ln\left(\frac{md}{\kappa}\right)}\right) + \frac{\sqrt{L}}{m}\cdot O\left(d\ln\left(\frac{md}{\kappa}\right)\right).\nonumber
\end{align}
Here the second last inequality and last inequality are by Lemma~\ref{claim:emp} and the fact that 
\begin{align}
    \mathrm{Pr}_{X\sim D}\left[ \max_{i}\bx_i^{\top} U_m^{-1}   \bx_i \geq L \right] \leq \mathbb{E}_{X\sim D} \left[   \min\{ \max_{i}x^{\top}U_m^{-1}   x   ,L  \}    \right] \leq \frac{1}{m}\cdot O\left(d\ln\left(\frac{md}{\kappa}\right)\right).\nonumber
\end{align}

In the case $n<m$, with similar arguments we have that
\begin{align}
    &  \mathbb{E}_{X\sim D}\left[ \min\left\{ \sqrt{\max_{i}\bx_i^{\top}\left(\Lambda+\frac{n}{m}\kappa \mathbf{I}\right)^{-1}   \bx_i} , \sqrt{L} \right\}         \right] \nonumber
  \\ & \leq \sqrt{\frac{36m}{n}}\mathbb{E}_{X\sim D}\left[ \min\left\{ \sqrt{\max_{i}\bx_i^{\top}U_m^{-1}   \bx_i} , \sqrt{Ln/m} \right\}   \right] \nonumber
  \\ & \leq  \sqrt{\frac{36m}{n}}\mathbb{E}_{X\sim D}\left[ \min\left\{ \sqrt{\max_{i}\bx_i^{\top}U_m^{-1}   \bx_i} , \sqrt{L} \right\}   \right]\nonumber
  \\ & \leq \sqrt{\frac{36m}{n}}\sqrt{\mathbb{E}_{X\sim D}\left[\min\left\{\max_{i}\bx_i^{\top}U_m^{-1}   \bx_i , L\right\}   \right]  } \nonumber
  \\ & \leq \sqrt{\frac{1}{n}}\cdot O\left(\sqrt{d\ln\left(\frac{md}{\kappa}\right)}\right)\nonumber
  \\ &\leq \sqrt{\frac{1}{n}}\cdot O\left(\sqrt{d\ln\left(\frac{md}{\kappa}\right)}\right) + \frac{\sqrt{L}}{m}\cdot O\left(d\ln\left(\frac{md}{\kappa}\right)\right).\nonumber
\end{align}
The proof is completed.

\section{Proof of the Matrix Concentration Inequality with Dynamic Upper Bounds (Lemma~\ref{lemma:dynamic-martingale-concentration})}\label{sec:pfdy}

In this section, we present the proof of our new matrix concentration inequality with dynamic upper bounds. For convenience, we first restate the inequality (Lemma~\ref{lemma:dynamic-martingale-concentration}) as follows.

\medskip
\noindent {\bf Lemma~\ref{lemma:dynamic-martingale-concentration} (restated).} {\it
Consider a sequence of stochastic PSD matrices $W_1, X_1, W_2, X_2, \dots, W_n, X_n \in \mathbb{R}^{d\times d}$. Let $\mathcal{F}_k = \sigma(W_1, X_1, W_2, X_2, \dots, W_{k-1}, X_{k-1})$ and $\mathcal{F}_k^+ = \sigma(W_1, X_1, W_2, X_2, \dots, W_{k-1}, X_{k-1}, W_k)$ be the natural filtration and $Y_k = \mathbb{E}[X_k |\mathcal{F}_k^+]$ for each $k \in \{1, 2, \dots, n\}$. Suppose $W_k$ is PD and increasing in $k$ (with respect to the semidefinite order) and $X_k \preccurlyeq  W_k$ for each $k$. For every $\delta > 0$ and $\epsilon\in (0,1)$, we have that
\begin{align*}
 \qquad \quad \quad \Pr\left[\sum_{k=1}^n X_k \preccurlyeq (1+\epsilon)\sum_{k=1}^{n} Y_k + \frac{4(\epsilon^2+2\epsilon+2)}{\epsilon}\ln((n+1)d/\delta) W_n\right] \geq 1 - \delta;  \qquad & & \eqref{eq:con1}
\\
\Pr\left[\sum_{k=1}^n X_k \succcurlyeq (1-\epsilon)\sum_{k=1}^{n} Y_k - \frac{4(\epsilon^2+2\epsilon+2)}{\epsilon}\ln((n+1)d/\delta) W_n\right] \geq 1 - \delta.  \qquad    & & \eqref{eq:con2}
\end{align*}
}

Before we start the proof, we introduce some basic properties of PSD matrices as below, whose proof is deferred to Appendix~\ref{sec:pffact}.
\begin{fact} \label{fact:matrix-loenwer-order} 
For any two PD matrices $A$ and $B$,  $A\preccurlyeq B$ is equivalent to each of the following inequalities, 
\begin{align}
    B^{-1} & \preccurlyeq A^{-1} , \label{eq:matrix-loenwer-order-1} \\
   A^{1/2}B^{-1}A^{1/2} & \preccurlyeq \mathbf{I} ,  \label{eq:matrix-loenwer-order-2} \\
    B^{-1/2} AB^{-1/2} & \preccurlyeq\mathbf{I}.  \label{eq:matrix-loenwer-order-3}
\end{align}
\end{fact}

We now start to prove Lemma~\ref{lemma:dynamic-martingale-concentration}, while the two helpful technical lemmas (Lemma~\ref{lemma:accc4} and Lemma~\ref{lemma:ac6}), are deferred to Section~\ref{sec:pfaa4} and Section~\ref{sec:pfac5}.
\begin{proof}[Proof of Lemma~\ref{lemma:dynamic-martingale-concentration}]
We first prove Equation~\eqref{eq:con1}.
For each $k \in \{0, 1, 2, \dots, n\}$, let $Z_k = \frac{4(\epsilon^2+2\epsilon+2)}{\epsilon}W_k$ and
\[
E_{k} := \mathrm{Tr}\left( \exp\left( Z_k^{-1/2}\left(\sum_{i=1}^k (X_i-(1+\epsilon)Y_i) \right)Z_k^{-1/2} \right) \right)    .
\]
By Lemma~\ref{lemma:accc4}, we have that $\mathbb{E}[E_n]\leq (n+1)d$. Therefore, by Markov inequality,
\begin{align*}
&\quad \Pr\left[\lambda_{\max}\left( Z_n^{-1/2}\left(\sum_{i=1}^n (X_i-(1+\epsilon)Y_i) \right)Z_n^{-1/2} \right)\geq C \right]\\
& \leq \Pr\left[\mathrm{Tr}\left(\exp\left( Z_n^{-1/2}\left(\sum_{i=1}^n (X_i-(1+\epsilon)Y_i) \right)Z_n^{-1/2} \right)\right)\geq e^{C} \right] 
\leq \frac{\mathbb{E}[E_n]}{e^{C}} \leq (n+1) de^{-C},
\end{align*}
which (by Fact~\ref{fact:matrix-loenwer-order}) means that
\begin{align}
\Pr\left[\sum_{i=1}^n (X_i-(1+\epsilon)Y_i) \preccurlyeq  CZ_n \right] \geq 1- (n+1)de^{-C}.\nonumber
\end{align}
Choosing $C = \ln((n+1)d/\delta)$ and recalling that $Z_n = \frac{4(\epsilon^2+2\epsilon+2)}{\epsilon}W_n$, we prove Equation~\eqref{eq:con1}.

We then prove Equation~\eqref{eq:con2}. First we define 
\begin{align}
    E_k' : = \mathrm{Tr}\left(  \exp\left( Z_k^{-1/2}\left(\sum_{i=1}^k ((1-\epsilon)Y_i-X_i) \right)Z_k^{-1/2} \right)\right) .\nonumber
\end{align}
Similarly, by Lemma~\ref{lemma:ac6}, we have that $\mathbb{E}[E_n']\leq (n+1)d$ and
\begin{align}
    &\quad \Pr\left[\lambda_{\max}\left( Z_n^{-1/2}\left(\sum_{i=1}^n ((1-\epsilon)Y_i-X_i) \right)Z_n^{-1/2} \right)\geq C \right]\\ \nonumber
& \leq \Pr\left[\mathrm{Tr}\left(\exp\left( Z_n^{-1/2}\left(\sum_{i=1}^n ((1-\epsilon)Y_i-X_i) \right)Z_n^{-1/2} \right)\right)\geq e^{C} \right] 
\leq \frac{\mathbb{E}[E'_n]}{e^{C}} \leq (n+1) de^{-C},\nonumber
\end{align}
which means that
\begin{align}
     \mathrm{Pr}\left[ \sum_{i=1}^n \left((1-\epsilon)Y_i -X_i \right) \preccurlyeq CZ_n \right]\geq 1-(n+1)de^{-C}. \nonumber
\end{align}
Choosing $C=\ln((n+1)d/\delta)$ we finish the proof.
\end{proof}

\begin{corollary}\label{coro1}
Given a sequence of stochastic random variables $X_1,X_2,\ldots, X_n$ such that $0\leq X_i \leq W$ for any $1\leq i \leq n$ with probability $1$. Let $\mathcal{F}_k = \sigma(X_1,X_2,\dots,X_{k-1})$ and $Y_k = \mathbb{E}[X_k|\mathcal{F}_k]$
For every $\delta>0$ and $\epsilon>0$, we have that
\begin{align}
    &\Pr\left[\sum_{k=1}^n X_k \leq (1+\epsilon)\sum_{k=1}^{n} Y_k + \frac{4(\epsilon^2+2\epsilon+2)}{\epsilon}\ln((n+1)/\delta) W \right] \geq 1 - \delta; \nonumber
 \\ & \Pr\left[\sum_{k=1}^n X_k \geq (1-\epsilon)\sum_{k=1}^{n} Y_k - \frac{4(\epsilon^2+2\epsilon+2)}{\epsilon}\ln((n+1)/\delta) W\right] \geq 1 - \delta. \nonumber
\end{align}
\end{corollary}
\begin{proof}
Letting $d=1$ and $W_k = W$ for $1\leq k \leq n$, by Lemma~\ref{lemma:dynamic-martingale-concentration} we finish the proof.
\end{proof}

\subsection{Statement and Proof of Lemma~\ref{lemma:accc4}}\label{sec:pfaa4}
\begin{lemma}\label{lemma:accc4}
For each $k \in \{1, 2, \dots, n\}$, we have that 
	\begin{align*}
	\mathbb{E}\left[ E_{k}|\mathcal{F}_k^+\right] \leq \mathbb{E}\left[E_{k-1}|\mathcal{F}_k^+\right] + d.
	\end{align*}
\end{lemma}

\begin{proof}
Firstly, we introduce a deep theorem from Lieb (Theorem 6, \citep{lieb1973convex}), which provides theoretical basis for a series of concentration inequalities on self-adjoint matrices.
\begin{theorem}\label{thm:lieb} Fix a $d$-dimensional symmetric matrix $H$. The function $f(A):= \mathrm{Tr}(\exp(\log(A) +H) )$ is concave on the $d$-dimensional positive definite cone.
\end{theorem}

Based on Theorem~\ref{thm:lieb}, \citep{tropp2012user} derived the corollary below.
\begin{corollary}\label{coro:tropp}
Fix a self-adjoint matrix $H$. Let $X$ be stochastic symmetric matrix
\begin{align}
    \mathbb{E}\left[\mathrm{Tr}\left(\exp(X+H) \right) \right]\leq  \mathrm{Tr}\left(   \mathbb{E}[e^{X}]+H\right).
\end{align}
\end{corollary}

Given Corollary~\ref{coro:tropp} we continue the analysis as below.
Note that $\mathbb{E}\left[E_{k-1}|\mathcal{F}_k^+\right]$ is a deterministic value. Throughout this proof, we will condition on $\mathcal{F}_k^+$. We calculate that
		\begin{align}
		&\quad \mathbb{E}\left[ E_{k}\right] \nonumber\\
		& = 	\mathbb{E}\left[ \mathrm{Tr}\left(\exp\left( Z_k^{-1/2}\left(\sum_{i=1}^k (X_i-(1+\epsilon)Y_i) \right)Z_k^{-1/2} \right) \right)  \right]  \nonumber
		\\ &  = \mathbb{E}\left[\mathrm{Tr}\left(\exp \left( Z_k^{-1/2}\left(\sum_{i=1}^{k-1} (X_i-(1+\epsilon)Y_i) \right)Z_k^{-1/2} + Z_k^{-1/2}(X_k-(1+\epsilon)Y_k)Z_k^{-1/2} \right)  \right) \right]\nonumber
		\\ & \leq \mathrm{Tr}\left(\exp\left( Z_k^{-1/2}\left(\sum_{i=1}^{k-1} (X_i-(1+\epsilon)Y_i) \right)Z_k^{-1/2} + \ln\left(\mathbb{E}\left[\exp\left(Z_k^{-1/2}(X_k-(1+\epsilon)Y_k)Z_k^{-1/2}\right) \right]\right) \right) \right) \label{eq:s0.5}
		\\ & \leq \mathrm{Tr}\left(\exp\left( Z_k^{-1/2}\left(\sum_{i=1}^{k-1} (X_i-(1+\epsilon)Y_i) \right)Z_k^{-1/2}  \right) \right)  \label{eq:s1}
		\\ & =\mathrm{Tr}\left(\exp\left( Z_k^{-1/2}Z_{k-1}^{1/2}Z_{k-1}^{-1/2}\left(\sum_{i=1}^{k-1} (X_i-(1+\epsilon)Y_i) \right)Z_{k-1}^{1/2}Z_{k-1}^{-1/2}Z_k^{-1/2} \right) \right)  \nonumber
		\\ & \leq  \mathrm{Tr}\left(\exp\left( Z_{k-1}^{-1/2}\left(\sum_{i=1}^k (X_i-(1+\epsilon)Y_i) \right)Z_{k-1}^{-1/2} \right) \right)+d\label{eq:s2}
		\\ & = E_{k-1}+d.\nonumber
	\end{align}
	
Here \eqref{eq:s0.5} is by Corollary~\ref{coro:tropp}, 
\eqref{eq:s1} is by Lemma~\ref{lemma:acccc5} (stated and proved in Section~\ref{sec:pfac5}) and the monotonicity of trace exponential with respect to the semidefinite order (see \citep{petz1994survey}, \S 2.2), and \eqref{eq:s2} is by Lemma~\ref{lemma:accc3} (stated and proved in Section~\ref{sec:pfaa3}, letting $U = Z_{k-1}^{1/2} Z_k^{-1/2}$ and $A = Z_{k-1}^{-1/2}\left(\sum_{i=1}^k (X_i-(1+\epsilon)Y_i) \right)Z_{k-1}^{-1/2}$, and one can verify that $U^\top U = Z_{k-1}^{1/2} Z_k^{-1} Z_{k-1}^{1/2} \preccurlyeq \mathbf{I}$ by Fact~\ref{fact:matrix-loenwer-order}).
\end{proof}

Similarly, we may establish the following lemma, whose proof is deferred to Appendix~\ref{sec:proof-lemma-ac6}.
\begin{lemma}\label{lemma:ac6}
For each $k \in \{1, 2, \dots, n\}$, we have that 
	\begin{align*}
	\mathbb{E}\left[ E'_{k}|\mathcal{F}_k^+\right] \leq \mathbb{E}\left[E'_{k-1}|\mathcal{F}_k^+\right] + d.
	\end{align*}
\end{lemma}

\subsection{Statement and Proof of Lemma~\ref{lemma:acccc5} }\label{sec:pfac5}
\begin{lemma}\label{lemma:acccc5}
For each $k \in \{1, 2, \dots, n\}$, we have that 
\[
\mathbb{E}\left[\exp\left(Z_k^{-1/2}(X_k-(1+\epsilon)Y_k)Z_k^{-1/2} \right) \Big|\mathcal{F}_k^+\right] \preccurlyeq \mathbf{I}.
\]
\end{lemma}
\begin{proof}
Throughout this proof, we will condition on $\mathcal{F}_k^+$. Let $U_k := Z_k^{-1/2}X_k Z_{k}^{-1/2}$ and $V_k :=  Z_k^{-1/2}Y_k Z_{k}^{-1/2}$. By the assumption in Lemma~\ref{lemma:dynamic-martingale-concentration} and our definition for $Z_k$, we have that  $0\preccurlyeq U_k \preccurlyeq  \frac{\epsilon}{4(\epsilon^2+2\epsilon+2)}\mathbf{I}$ and $\mathbb{E}[U_k ]=V_k$. Therefore, $0\preccurlyeq V_k \preccurlyeq  \frac{\epsilon}{4(\epsilon^2+2\epsilon+2)}\mathbf{I}$. We now compute that
\begin{align}
& \quad  \mathbb{E}\left[\exp\left(Z_k^{-1/2}(X_k-(1+\epsilon)Y_k)Z_k^{-1/2} \right)\right] \nonumber\\ 
& =  \mathbb{E}\left[\exp\left(U_k-(1+\epsilon)V_k\right)\right] \nonumber\\ 
&  = \mathbb{E}\left[ \mathbf{I} +(U_k-(1+\epsilon)V_k) +\frac{1}{2}(U_k-(1+\epsilon)V_k)^2 + \sum_{i\geq 3}\frac{1}{i!} (U_k-(1+\epsilon)V_k)^i \right] \label{eq:qwe1}\\ 
& \leq \mathbb{E}\left[ \mathbf{I} +(U_k-(1+\epsilon)V_k) +\frac{1}{2}(U_k-(1+\epsilon)V_k)^2 + \sum_{i\geq 3}\frac{1}{i!} (U_k-(1+\epsilon)V_k)^2 \right] \label{eq:qwe2}\\ 
& \preccurlyeq \mathbb{E}\left[\mathbf{I} +(U_k-(1+\epsilon)V_k) +2(U_k-(1+\epsilon)V_k)^2 \right] \nonumber\\ 
& \preccurlyeq \mathbb{E}\left[\mathbf{I} +(U_k-(1+\epsilon)V_k) +4U_k^2 + 4(1+\epsilon)^2V_k^2 \right] \label{eq:qwe3}\\ 
& \preccurlyeq \mathbb{E}\left[\mathbf{I} +(U_k-(1+\epsilon)V_k) +\frac{4\epsilon}{4(\epsilon^2+2\epsilon+2)}U_k + \frac{4(1+\epsilon)^2\epsilon}{4(\epsilon^2+2\epsilon+2)}V_k \right]\label{eq:qwe4}\\ 
&  = \mathbf{I}.\nonumber
\end{align}
Here, \eqref{eq:qwe1} is by Taylor series expansion, \eqref{eq:qwe2} is by the fact that $-\mathbf{I}\preccurlyeq  (U_k-(1+\epsilon)V_k)\preccurlyeq \mathbf{I}$ and $ X^k -X^2 = X(X^{k-2}-\mathbf{I})X\preccurlyeq 0$ for $k\geq 2$ and $-\mathbf{I}\preccurlyeq X\preccurlyeq \mathbf{I}$. \eqref{eq:qwe3} is because $4U_k^2 +4(1+\epsilon)^2V_k^2 - 2(U_k-(1+\epsilon)V_k)^2 = 2(U_k+(1+\epsilon)V_k)^2\succcurlyeq 0$, and \eqref{eq:qwe4} is by the fact that $\frac{4(\epsilon^2+2\epsilon+2)}{\epsilon}U_k^2-U_k  = U_k^{1/2}\left(\frac{4(\epsilon^2+2\epsilon+2)}{\epsilon}U_k-\mathbf{I}\right)U_k^{1/2}\preccurlyeq 0$ and $\frac{4(\epsilon^2+2\epsilon+2)}{\epsilon}V_k^2-V_k  = V_k^{1/2}\left(\frac{4(\epsilon^2+2\epsilon+2)}{\epsilon}V_k-\mathbf{I}\right)V_k^{1/2}\preccurlyeq 0$.
\end{proof}

\subsection{Statement and Proof of Lemma~\ref{lemma:accc3}}\label{sec:pfaa3}

\begin{lemma} \label{lemma:accc3}
Let $A$ be a real symmetric matrix. For any $U$ such that $U^\top U \preccurlyeq \mathbf{I}$, it holds that
\begin{align}
\mathrm{Tr}(\exp(U^{\top} A U ))  \leq \mathrm{Tr}(\exp(A))+d.
\end{align} 
\end{lemma}
\begin{proof}
We can assume without loss of generality that $A =\mathrm{diag}(\lambda_1,\lambda_2,\dots ,\lambda_d)$ and $\lambda_d\leq \lambda_{d-1}\leq \dots \leq \lambda_{\ell - 1} < 0 \leq \lambda_{\ell}\leq \dots \leq \lambda_1$.

Let us write $U^\top = Q^\top L^\top$ to be its QR decomposition where $Q^\top$ is orthogonal and $L^\top$ is an upper triangular matrix (and therefore $L$ is a lower triangular matrix). We claim that \begin{align}\label{eq:lem-accc3-1}
U^\top \mathbf{I}_j U \preccurlyeq Q^\top \mathbf{I}_j Q,
\end{align}
where we define $\mathbf{I}_j :=  \mathrm{diag}(0, \dots, 0, 1, \dots 1)$ to be a rank $j$ matrix. Note that $U^\top \mathbf{I}_j U  = Q^\top L^\top \mathbf{I}_j L Q$. Therefore, to prove the claim, it suffices to show that $ L^\top \mathbf{I}_j L \preccurlyeq \mathbf{I}_j$. Since $U^\top U \preccurlyeq \mathbf{I}$, we also have that $L^\top L = Q U^\top U Q^\top \preccurlyeq \mathbf{I}$. Note that for any vector $x$, it holds that $x^\top (L \mathbf{I}_j)^\top (L \mathbf{I}_j) x = (\mathbf{I}_j x)^\top L^\top L (\mathbf{I}_j x) \leq x^\top \mathbf{I}_j x$. Therefore we conclude that $L^\top \mathbf{I}_j L = (L \mathbf{I}_j)^\top (L \mathbf{I}_j) \preccurlyeq \mathbf{I}_j$, where the first equality holds because  $L$ is a lower triangular matrix.

Let $\tilde{\lambda}_i := \lambda_i - \lambda_{i - 1} \geq 0$ for each $i \in \{1, 2, \dots, \ell - 1\}$, $\tilde{\lambda}_\ell := \lambda_\ell \geq 0$, and $\tilde{\lambda}_j := 0$ for each $j \in \{\ell + 1, \ell + 2, \dots, d\}$. By \eqref{eq:lem-accc3-1}, we then have that
\[
U^\top A U  \preccurlyeq \sum_{j=1}^d \tilde{\lambda}_j U^\top \mathbf{I}_j U \preccurlyeq  \sum_{j=1}^d \tilde{\lambda}_j Q^\top \mathbf{I}_j Q.
\]

By the monotonicity of trace exponential with respect to the semidefinite order (see, e.g. \citep{petz1994survey} \S 2.2), we have that
\begin{align*}
\mathrm{Tr}(\exp(U^\top A U)) &\leq \mathrm{Tr}\left(\exp\left( \sum_{j=1}^d \tilde{\lambda}_j Q^\top \mathbf{I}_j Q\right)\right)\\
&= \sum_{j=1}^{d} \exp\left( \sum_{i=j}^{d} \tilde{\lambda}_i\right) \leq \sum_{j=1}^d \exp(\lambda_i) + d = \mathrm{Tr}(\exp(A)) + d .
\end{align*}

\end{proof}

\section{Regret Lower Bound for Context-blind Batch Learning}\label{sec:low}
In this section, we prove Lemma~\ref{lemma:lbdynamic}, which implies the regret lower bound theorem (Theorem~\ref{thm:lb}).

 \begin{lemma}\label{lemma:lbdynamic} 
For any algorithm $\mathcal{G}$ with batch complexity $M$, assuming $T\geq d\log_2(K)$, the minimax regret is at least 
\begin{align}
 \Omega\left( \frac{1}{\mathrm{poly}\ln(Td)}\min\left\{
 \begin{aligned} & T^{\frac{1}{2-2^{-M+1}}}(d\log_2(K))^{\frac{1-2^{-M+1}}{2-2^{-M+1}}} \left(\frac{\min\{K,d\}}{\log_2(K)}\right)^{\frac{2^{-M+1}}{2-2^{-M+1}}}  \\ & T^{\frac{1}{2-2^{-M+2}}} (d\log_2(K))^{\frac{1-2^{-M+2}}{2-2^{-M+2}}}  \end{aligned}\right\} \right). \nonumber 
\end{align}
\end{lemma}

\begin{proof}

 Define $\bar{d}=d/M$ and  $\tilde{h} = \frac{1}{2}\min\{K, \bar{d} \}$. Define $\check{d} = \frac{\bar{d}\log_2(K)}{M\ln(dM)}\leq d\log_2(K)\leq d$ and $\check{h} = \frac{\tilde{h}}{\log_2(K)}$. 
When $\check{d}\leq T\leq \check{d} \check{h}^{2-2^{-M+2}}$, define 
 $\gamma := T^{\frac{1}{2-2^{-M+2}}}\check{d}^{\frac{1-2^{-M+2}}{d^{2-2^{-M+2}}}} $. When $T> \check{h}^{2-2^{-M+2}} $, define $\gamma:= T^{\frac{1}{2-2^{-M+1}}}\check{d}^{\frac{1-2^{-M+1}}{2-2^{-M+1}}}\check{h}^{\frac{2^{-M+1}}{2-2^{-M+1}}}\geq \check{d} $. Given $\gamma$, we further define 
 $T_1 =\gamma $, $T_2 =\max\{\gamma, \gamma\cdot \sqrt{\frac{T_1\ln(dM)}{\bar{d}h}}\} $ and $
     T_k =\gamma \cdot \sqrt{\frac{MT_{k-1}\ln(dM)}{\bar{d}\log_2(K)}}$ for any $ 3\leq k \leq M$.
   Let $\epsilon_1 = \frac{1}{100}$, $\epsilon_2 = \min\left\{ \frac{1}{100}\sqrt{\frac{\bar{d}^2}{T_1\ln(dM)}} , 1\right\}$ and $\epsilon_k= \frac{1}{100} \sqrt{\frac{\bar{d}\log_2(K)}{MT_{k-1}\ln(dM)}}$ for $k = 3,4,\ldots, M$.

Fix the algorithm $\mathcal{G}$. Below we assume the randomness of $\mathcal{G}$ is considered in the expectation operator $\mathbb{E}[\cdot]$ and probability operator $\mathrm{Pr}[\cdot]$.  Let $\mathrm{Regret}_{\theta,D}(T)$ be the expected regret under $\mathcal{G}$ with hidden parameter as $\theta$ and context distribution as $D$. We aim to design $D$ and $\theta$ such that $\mathrm{Regret}_{\theta, D}(T)$ is large enough.

 Without loss of generality, we assume $d/M$ is an integer. In the construction below, we divide $\mathbb{R}^{d}$ into $M$ subspaces, where for each subspace we construct a hard case.
We let $\mathbb{R}^d = \otimes_{i=1}^{M}\mathcal{U}^i$, where $U^i$ is the subspace spanned by $\{\mathbf{e}_{(i-1)d/M +j} \}_{j=1}^{d/M}$. 
In each round the noise is set to be Gaussian with variance $1$. 
The hidden parameter $\theta=\{\theta^i\}_{i=1}^M$ is chosen from the space $\mathcal{Y} = \otimes_{i=1}^{M}\mathcal{Y}^i$. The context distribution $D$ is given by the average of $\{D^i\}_{i=1}^M$. That is, in the $t$-th  round, the environment first sample $u_t\sim \mathrm{Uniform}(\{1,2,\ldots,M\})$, and then sample $X$ according to  $D_{u_t}$.  
In words, we divide the original problem into $M$ independent sub-problems. For the  $k$-th sub-problem is $d/M$, the dimension is $d/M$, the parameter space is $\mathcal{Y}^k$ and the context distribution is $D_k$.  

We now define $\mathcal{Y}^k$ and $D_k$.
 Let $\bar{d}=d/M$. 
We have two cases: Case \uppercase\expandafter{\romannumeral1}: $k=2$;  and Case \uppercase\expandafter{\romannumeral2}: $k =1 $ or  $3\leq k \leq M$.
Below we respectively define the parameter space  $\mathcal{Y}^k$ for $\theta^k$ and the context distribution  $D_{k}$.


\paragraph{Case \uppercase\expandafter{\romannumeral1}: $k= 2$.}
Define $\mathcal{Y}^2 = \left\{-\epsilon_2,0,\epsilon_2\right\}^{\bar{d}}$. Recall that $\tilde{h} = \frac{1}{2}\min\{K,\bar{d}\}$. For $\mathcal{V}\subset [2\tilde{h}]$ with $|\mathcal{V}|=\tilde{h}$, we define the context distribution $D_2(\mathcal{V})$ by letting $\mathrm{Pr}_{X\sim D_2(\mathcal{V})}[X = \{\mathbf{e}_j\}_{j\in \mathcal{V}} \cup \{\mathbf{e}_i\}]= \frac{1}{\bar{d}-\tilde{h}}$ for any $i \in [\bar{d}]/\mathcal{V}$. In words, the sub-problem for the second batch is a contextual bandit problem with $\tilde{h}+1$ arms. Among the $h+1$ arms, there are $\tilde{h}$ arms which appear in each round, and the left arms appears with equal probability as $\frac{1}{\bar{d}-\tilde{h}}$. 

In the case $K \gg \tilde{h}+1$, we simply repeat the first arm for $K -(\tilde{h}+1) $ times to construct the $K$-armed linear bandit problem. 
Without loss of generality, we still use $D_2(\mathcal{V})$ to denote this context. 
For fixed $\mathcal{V}$ and $\xi\in \{-1,1\}^{\bar{d}}$, we further define $\theta^2(\mathcal{V},\xi)$ by setting $\theta^2_i(\mathcal{V},\xi) =0$ for $i \in \mathcal{V}$ and $\theta^2_i(\mathcal{V},\xi) = \xi_i \epsilon_2$ for $i\in [\bar{d}]/\mathcal{V}$.

Since the sub-problem is a contextual bandit problem, we could view the context $X$ as a subset of $[\bar{d}]$.
 Given a permutation $\sigma\in S_{\bar{d}}$ and  a context $X\subset [\bar{d}]$, we define the context vector $\sigma(X) = \{\sigma(i)\}_{i\in X}$.  With a slight abuse of notations, we use $D_2(\sigma)$ as the shorthand for $D_2(\sigma([1,2,\ldots,\tilde{h}] ))$.

Let the policy $\pi_{1}$ be the policy for the first batch. 
Clearly, $\pi_{1}$ is independent of the context distribution. Let the distribution of $\pi_{1}$ be $\Pi_{1}$. We then claim the following lemma.
\begin{lemma}\label{lemma:m1}
    \begin{align}
       \min_{\sigma\in S_{\bar{d}}} \mathbb{E}_{\pi_1\sim \Pi_{1}}\mathbb{E}_{X\sim D(\sigma)}\left[\mathbb{I}\left[\pi_1(X)\notin \{\sigma(1),\sigma(2),\ldots,\sigma(\tilde{h})\} \right] \right] \leq \frac{1}{\tilde{h}+1}. \nonumber
    \end{align}
\end{lemma}
\begin{proof}
Let $\mathcal{A}$ be the set of subsets of $[\bar{d}]$ with size $\tilde{h}+1$. Let $Z_i = \{1,2,\ldots,\tilde{h}, \tilde{h}+i\}$ for $1\leq i \leq \bar{d}-\tilde{h}$.
Note that for any $1\leq i\leq \bar{d}-\tilde{h}$ and $\sigma\in S_{\bar{d}}$, $\pi(\sigma(Z_i))\notin \{\sigma(1),\sigma(2),\ldots,\sigma(\tilde{h}) \}$ implies that  $\pi(X)=\sigma(\tilde{h}+i)$. Then we have that
\begin{align}
& \sum_{\sigma\in S_{\bar{d}}}\frac{1}{\bar{d}!}\mathbb{E}_{\pi_1\sim \Pi_{1}}\mathbb{E}_{X\sim D_2(\sigma)}\left[\mathbb{I}\left[\pi_1(X)\notin \{\sigma(1),\sigma(2),\ldots,\sigma(\tilde{h})\} \right] \right] \nonumber
\\ & =\frac{1}{ \bar{d}!(\bar{d}-\tilde{h})} \sum_{X\in \mathcal{A}} \sum_{\sigma\in S_{\bar{d}}}\sum_{i=1}^{d_1} \mathbb{I}[\sigma^{-1}(X)=Z_i ] \mathbb{E}_{\pi_{1}\sim \Pi_{1}}\left[\mathbb{I}\left[\pi_1(X)= \sigma(\tilde{h}+i)\right]\right] \nonumber
\\ & = \frac{1}{ \bar{d}!(\bar{d}-\tilde{h})} \sum_{X\in \mathcal{A}} \mathbb{E}_{\pi_{1}\sim \Pi_{1}}\left[\sum_{i=1}^{\bar{d}-\tilde{h}} \sum_{\sigma\in S_{\bar{d}}} \mathbb{I}[\sigma^{-1}(X)=Z_i, \pi_1(X)= \sigma(h+i) ]\right] \nonumber
\\ &  = \frac{1}{ \bar{d}!(\bar{d}-\tilde{h})} \sum_{X\in \mathcal{A}} \mathbb{E}_{\pi_{1}\sim \Pi_{1}}\left[\sum_{i=1}^{d_1} \sum_{\sigma\in S_{\bar{d}},x\in X} \mathbb{I}[\sigma^{-1}(X)=Z_i, \pi_1(X)= x, x = \sigma(\tilde{h}+i) ]\right] \nonumber
\\ & =\frac{1}{\bar{d}!(\bar{d}-\tilde{h})} \sum_{X\in \mathcal{A},x\in X}\mathrm{Pr}_{\pi_1\sim \Pi_1}\left[\pi_1(X)=x\right]\cdot \sum_{\sigma\in S_{\bar{d}}}\sum_{i=1}^{\bar{d}-\tilde{h}} \mathbb{I}[\sigma^{-1}(X)=Z_i, x= \sigma(\tilde{h}+i)]\nonumber
\\ & = \frac{\tilde{h}!(\bar{d}-\tilde{h}-1)!}{\bar{d}!} \sum_{X\in \mathcal{A},x\in X}\mathrm{Pr}_{\pi_1\sim \Pi_1}\left[\pi_1(X)=x\right]\nonumber
\\ &= \frac{\tilde{h}!(\bar{d}-\tilde{h}-1)!}{\bar{d}!}\sum_{X\in \mathcal{A}} 1\nonumber
\\ &= \frac{\tilde{h}!(\bar{d}-\tilde{h}-1)!}{\bar{d}!}|\mathcal{A}|\nonumber
\\ &  = \frac{1}{\tilde{h}+1}.\label{eq:m1}
\end{align}
Then the conclusion follows easily.
\end{proof}

Without loss of generality, we suppose that the identical permutation satisfies the condition in Lemma~\ref{lemma:m1}, i.e.,
\begin{align}
   \mathbb{E}_{\pi_1\sim \Pi_1} \mathbb{E}_{X\sim D_2([\tilde{h}])}\mathbb{I}\left[ \pi_1(X)\notin \{1,2,\ldots,\tilde{h}\} \right] \leq \frac{1}{\tilde{h}+1}.\label{eq:m2}
\end{align}

Define $J = \{i \geq \tilde{h}+1 | \mathbb{E}_{\pi_1\sim \Pi_1} \mathbb{E}_{X\sim D}\mathbb{I}[\pi(X)=i]\leq \frac{2}{(\tilde{h}+1)(\bar{d}-\tilde{h})} \}$, then by \eqref{eq:m2} we have that $|J|\geq \frac{\bar{d}-\tilde{h}}{2}$.  Without loss of generality, we assume $J = \{\tilde{h}+1,\tilde{h}+2,\ldots,\tilde{h}+\ell \}$ where $\ell\geq \frac{\bar{d}-\tilde{h}}{2}$ is the size of $J$.
Then the context distribution is fixed as $D_2=D_2([\tilde{h}])$.

\paragraph{Case \uppercase\expandafter{\romannumeral2}: $k=1$  or $3\leq k \leq M$.}  We consider to construct the hard case for $K$-armed linear contextual bandit problem.  We assume $a= \min\{\log_2(K),\bar{d}\}$ and $b = \bar{d}/a$ are both integers. 
Define $\mathcal{Y}^k = \epsilon_k \cdot  \{-1,1\}^{\bar{d}}$. The context distribution $D^k$ is defined as the uniform distribution over  $X^{k,i} = \frac{1}{a} \cdot  \{ \bx: \bx_i = 0, \forall i\notin [(i-1)a+1,i  a ], \bx_{i}\in \{-1,1\}, \forall i \in  [(i-1)a+1,i a ]     \}$ for $1\leq i \leq b$.

Now we start to analyze the minimax lower bound over all $\theta \in \mathcal{Y}$ and the context distribution $D$ described above. 
Let $\{t_{i}\}_{i=1}^M$ be the time schedule by running the algorithm $\mathcal{G}$.  Then there exists $1\leq i \leq M$ such that $\sum_{j\leq i-1}t_j \leq \sum_{j\leq i-1}T_j <\sum_{j\leq i}T_{j}\leq\sum_{j \leq i} t_{j}$, where we define $T_0 = t_0 = 0$.
Denote $\mathcal{E}_i$ be the event where $i$ is the smallest number such that $\sum_{j\leq i}t_i \leq \sum_{j\leq i}T_i <\sum_{j\leq i+1}T_{i+1}\leq\sum_{j \leq i+1} t_{i+1}$. Then it holds that $\sum_{i=1}^M \mathbb{I}[\mathcal{E}_i]=1$.

To proceed, we have the lemma as below.
\begin{lemma}\label{lemma:lbnew} Let $\pi_k$ be the policy for the $k$-th batch and $\pi_{(t)}$ be the policy for the $t$-th step. 
   Fix $j \in [\tilde{h}+1 , \tilde{h}+\ell]$. For $\theta \in \mathcal{Y}$, we define 
   \begin{align}
       R_j^k(\theta) = \left\{  \begin{aligned}& \frac{1}{\bar{d}-\tilde{h}} \sum_{t=T_1+1}^{T_1+T_2} \mathbb{E}_{\theta,\bx\sim \pi_{(t)}([\tilde{h}]\cap \{\tilde{h}+j\})}\left[ \mathbb{I}[\bx = j, \theta_j^2<0]+ \mathbb{I}[\bx\neq j, \theta_j^2>0]  \right] \cdot\epsilon_2,      \quad  k =2  ;\\ & \frac{1}{\bar{d}}\sum_{t=\mathcal{T}_{k-1}+1}^{\mathcal{T}_k} \mathbb{E}_{\theta, \bx\sim \pi_{(t)}(X^{k,\left\lceil j/a  \right\rceil})}\left[ \mathbb{I}[\bx_j \theta^k_j <0] \right]\epsilon_k,  \qquad\qquad  \quad       \quad  k=1 \text{ or } 3\leq k \leq M.\end{aligned} \right.
   \end{align}
   and $R_j(\theta) = \sum_{k=1}^M R_j^k(\theta)$.
   Then for any $\tilde{h}+1 \leq j \leq \tilde{h}+\ell$, we have that
    \begin{align}
        \frac{1}{|\mathcal{Y}|}\sum_{\theta\in \mathcal{Y} }R_j(\theta)\geq \frac{1}{8(M+1)}\min\left\{\min_{k\in [M],k\neq 2} \frac{T_k \epsilon_k}{\bar{d}}, \frac{T_2\epsilon_2}{8\bar{d}}\right\} .\nonumber
    \end{align}
\end{lemma}
 
With Lemma~\ref{lemma:lbnew} in hand, noting that $\ell\geq \frac{\bar{d}-\tilde{h}}{4}\geq \frac{\bar{d}}{8}$, we  obtain that
\begin{align}
     \frac{1}{|\mathcal{Y}|}\sum_{j=h+1}^{h+\ell}\sum_{\theta\in \mathcal{Y} }R_j(\theta)\geq \frac{1}{32(M+1)}\min\left\{\min_{k\in [M],k\neq 2} T_k \epsilon_k, \frac{T_2\epsilon_2}{8}\right\} .\nonumber
\end{align}

Noting that the expected regret under parameter $\theta$ is at least $\mathrm{Regret}_{\theta,D}(T) \geq \sum_{j=\tilde{h}+1}^{\tilde{h}+l}R_j(\theta)$, we learn that
\begin{align}
    \frac{1}{|\mathcal{Y}|}\sum_{\theta \in \mathcal{Y}} \mathrm{Regret}_{\theta,D}(T)  & \geq  \frac{1}{|\mathcal{Y}|}\sum_{j=\tilde{h}+1}^{\tilde{h}+\ell}\sum_{\theta\in \mathcal{Y} }R_j(\theta) \geq  \Omega\left( \gamma/M\right)
\end{align}

Then there exists some $\theta^*\in \mathcal{Y}$ such that $\mathrm{Regret}_{\theta^*,D}(T)\geq \Omega(\gamma/M)$. The proof is completed by definition of $\gamma$.

\end{proof}

It remains to prove Lemma~\ref{lemma:lbnew}.
\begin{proof}[Proof of Lemma~\ref{lemma:lbnew}]
Recall the definition of $\mathcal{E}_k$ for $1\leq k \leq M$. We further define $\mathcal{H}_k$ as below.
\begin{itemize}
\item $\mathcal{H}_1 = \bar{\Omega}$, where $\bar{\Omega}$ is the entire probability space;
\item  $\mathcal{H}_2= \{ n_j\leq \max\{\frac{6T_1}{(h+1)(\bar{d}-h)M},6\ln(dM)  \}, \forall  h+1\leq j \leq h+l         \}$, where $n_j$ denotes the number of times the $j$-th arm is taken in the first batch;
\item $\mathcal{H}_k = \{ n^k_i \leq \max\{\frac{6MT_{k-1}}{b} ,6\ln(dM)\}, \forall 1\leq i\leq b         \}$ for $3\leq k \leq M$, where $n^k_i = \sum_{t=1}^{\mathcal{T}_{k-1}}\mathbb{I}[X_t = X^{k,i}]$, i.e., the number of times when the context is $X^{k,i}$.
\end{itemize}

Since the $\pi^1$ is independent of $\theta$, then $\{\mathcal{H}_k\}_{k=1}^M$ is also independent of $\theta$. Using Lemma~\ref{lemma:con}, it is easy to show that $\mathrm{Pr}_{\theta}[\mathcal{H}_k] \geq 1-\frac{1}{10M}$ for any $1\leq k \leq M$ and any $\theta \in \mathcal{Y}$.  

Let $p_j^k(\theta) = \mathbb{E}_{\theta, \bx\sim \pi_k(X^{k,\left\lceil j/a  \right\rceil})}\left[ \mathbb{I}[\bx_j \theta^k_j <0] |\mathcal{E}_k \cap \mathcal{H}_k\right]$ for $k = 1$ or $3\leq k \leq M$. For $k = 2$, we define $p_j^k(\theta) =\mathbb{E}_{\theta,\bx\sim \pi_2([\tilde{h}]\cap \{\tilde{h}+j\})}\left[ \mathbb{I}[\bx = j, \theta_j^2<0]+ \mathbb{I}[\bx\neq j, \theta_j^2>0]  |\mathcal{E}_2\cap \mathcal{H}_2\right]$. 

Recall that $\mathcal{F}_{t}$ denotes the event field over the first $t$ steps. 
For fixed $\theta\in \mathcal{Y}$, we denote $\theta^k_j$ be the vector in $\mathcal{Y}$ by reflecting the $j$-th dimension of $\theta^k$. Using Pinsker's inequality (Lemma~\ref{lemma:pinsker}), and noting that $\mathcal{E}_k\cap \mathcal{H}_k$ are measurable with respect to $\mathcal{F}_{\mathcal{T}_{k-1}}$. 
\begin{align}
    p_j^k(\theta) + p_j^k(\theta^k_j)\geq  1-  \sqrt{\frac{1}{2} D_{\mathrm{KL}}\left(\mathrm{Pr}^{\mathcal{T}_{k-1}}_{\theta}[\cdot |\mathcal{E}_k\cap \mathcal{H}_k],\mathrm{Pr}^{\mathcal{T}_{k-1}}_{\theta^k_j}[\mathcal{E}_k\cap \mathcal{H}_k] \right)} \nonumber
\end{align}
for $k = 1$ or $3\leq k \leq M$, where $\mathrm{Pr}^{t}[\cdot]$ denotes probability distribution over the first $t$ steps. By definition of $\mathcal{E}_k$ and $\mathcal{H}_k$, we have that
\begin{align}
    D_{\mathrm{KL}}\left(\mathrm{Pr}^{\mathcal{T}_{k-1}}_{\theta}[\cdot |\mathcal{E}_k\cap \mathcal{H}_k],\mathrm{Pr}^{\mathcal{T}_{k-1}}_{\theta^k_j}[\mathcal{E}_k\cap \mathcal{H}_k] \right)\leq \max\left\{\frac{6MT_{k-1}}{b}, 6\ln(dM)\right\}\frac{4\epsilon_k^2}{a^2}\leq \frac{1}{8}. \nonumber
\end{align}
As a result, we have that 
\begin{align}
p_j^k(\theta) + p_j^k(\theta^k_j)\geq \frac{3}{4}.\label{eq:pp0}
\end{align}

Using similar arguments, and noting that $n_j \leq \max\{\frac{6T_1}{(\tilde{h}+1)(\bar{d}-\tilde{h})M},6\ln(dM)  \}$ conditioned on $\mathcal{H}_2$,   we have that
\begin{align}
    p_j^2(\theta) + p_j^2(\theta_j^k) \geq 1- \sqrt{\frac{1}{2} D_{\mathrm{KL}}(\mathrm{Pr}^{T_1}_{D_{2}, \theta) }[\cdot |\mathcal{H}_2\cap \mathcal{E}_2 ] ,  \mathrm{Pr}^{T_1}_{D_{2}, \theta^2_j }[\cdot|\mathcal{H}_2\cap \mathcal{E}_2 ] )}\geq \frac{3}{4}.\label{eq:pp1}
\end{align}

Let $\zeta_k = \frac{1}{\bar{d}}T_k\epsilon_k$ for $k=1$ or $3\leq k \leq M$ and $\zeta_2 = \frac{1}{8\bar{b}}T_2\epsilon_2$.
 By the definition of $R_j(\theta)$, and noting that $\ell\geq \frac{\bar{d}}{8}$, we have that 
 \begin{align}
    &  R_j(\theta)\geq  \sum_{k=1}^M \mathrm{Pr}_{\theta}[\mathcal{E}_k\cap\mathcal{H}_k]p_j^k(\theta) \zeta_k \nonumber
    \\ & R_j(\theta^k_j)= \sum_{k=1}^M \mathrm{Pr}_{\theta^k_j}[\mathcal{E}_k\cap\mathcal{H}_k]p_j^k(\theta_j^k)\zeta_k.\nonumber
 \end{align}
using Pinsker's inequality (Lemma~\ref{lemma:pinsker}), for any $1\leq k \leq M$, it holds that
\begin{align}
\left|    \mathrm{Pr}_{\theta}[\mathcal{E}_k\cap \mathcal{H}_k] -\mathrm{Pr}_{\theta_k^j}[\mathcal{E}_k\cap \mathcal{H}_k]\right| \leq  \sqrt{\frac{1}{2} D_{\mathrm{KL}}\left(\mathrm{Pr}^{\mathcal{T}_{k-1}}_{\theta}[\cdot |\mathcal{E}_k\cap \mathcal{H}_k],\mathrm{Pr}^{\mathcal{T}_{k-1}}_{\theta^k_j}[\mathcal{E}_k\cap \mathcal{H}_k] \right)} \leq \frac{1}{2M}.\nonumber \nonumber
\end{align}

Therefore, we have that
 \begin{align}
    R_j(\theta)+ \sum_{k=1}^M R_j(\theta_j^k)\geq \sum_{k=1}^M \max\left\{\frac{3}{4}\mathrm{Pr}_{\theta}[\mathcal{E}_k \cap \mathcal{H}_k]  - \frac{1}{2M},0\right\} \zeta_k \geq  \frac{1}{8}\min_{k}\zeta_k.
 \end{align}
 Taking sum over $\mathcal{Y}$, we have that 
 \begin{align}
     \frac{1}{|\mathcal{Y}|}\sum_{\theta \in \mathcal{Y}}\sum_{k =1 }^M \mathrm{Pr}_{\theta}[\mathcal{E}_k\cap \mathcal{H}_k]p_j^k(\theta) T_k \epsilon_k \geq \frac{1}{8(M+1)}\min_{k}\zeta_k.\nonumber  \end{align}
 The proof is completed.

\end{proof}

\section{Extension to the Context-Aware Case}\label{sec:context-aware}

\subsection{The Batch Learning Algorithm}
In the context-aware case, we can observe the context  before determining the  policy. For each $1\leq k \leq M$ we can learn $\pi^k$ using $\{X_t\}_{t = \mathcal{T}_{k-1}+1}^{\mathcal{T}_{k}}$. In particular, by observing the context of the first batch, we can learn a design policy which is better than $\pi^{\mathtt{G}}$. As a consequence, the final regret upper bound would be smaller than the context-blind case.

The algorithm is presented in Algorithm~\ref{alg:context-aware}. At the start of the $k$-th batch, we observe the context $\{X_t\}_{t=\mathcal{T}_{k-1}+1}^{\mathcal{T}_k}$ and then play elimination by previous information to get the eliminated context $\{X_t^{(k)}\}_{t=\mathcal{T}_{k-1}+1}^{\mathcal{T}_k}$ . After that, we run $\mathtt{ExplorationPolicy}$ with input as $\{X_t^{(k)}\}_{t=\mathcal{T}_{k-1}+1}^{\mathcal{T}_k}$ to search the near-optimal design policy. Since the $\{X_t\}_{t=\mathcal{T}_{k-1}+1}^{\mathcal{T}_k}$ is independent of all the histories before the $k$-th batch, there is no need to break the samples in the $k$-th batch to two parts to keep the independence.


\begin{algorithm}
\caption{{\tt Context-Aware Batch Learning}} \label{alg:context-aware}
\begin{algorithmic}[1]
\STATE{\textbf{Initialize:} $\lambda\leftarrow 10/T$;  $\Lambda_0 \leftarrow \lambda \mathbf{I};$ $\hat{\btheta}_0\leftarrow \boldsymbol{0}$;}

\FOR{$k=1,2,\ldots,M$}
\STATE{Observe the context $\{X_t\}_{t=\mathcal{T}_{k-1}+1}^{\mathcal{T_k}}$;}
\FOR{$t= \mathcal{T}_{k-1}+1, \mathcal{T}_{k-1}+2, \ldots , \mathcal{T}_k$}
\STATE{$X_t^{(k)}\leftarrow \mathcal{E}(X_t,\{\Lambda_{i},\hat{\btheta}_{i}\}_{i=1}^{k-1} )$; }
\ENDFOR
\STATE{$\pi_{k}\leftarrow \mathtt{ExplorationPolicy}( \{X_t^{(k)}   )\}_{t=\mathcal{T}_{k-1}+1}^{\mathcal{T}_k} )$; }
\FOR{$t= \mathcal{T}_{k-1}+1, \mathcal{T}_{k-1}+2, \ldots , \mathcal{T}_k$}
\STATE{Play the arm with the feature vector $\boldsymbol{y}_t \sim \pi_{k}(X^{(k)}_t)$ and receive the reward $r_t$;}
\ENDFOR
\STATE{$\Lambda_{k}\leftarrow \lambda \mathbf{I} +  \sum_{t=\mathcal{T}_{k-1}+1}^{\mathcal{T}_{k-1}+T_k/2} \boldsymbol{y}_t \boldsymbol{y}_t^{\top}$; $\hat{\btheta}_k \leftarrow \Lambda_k^{-1}  \sum_{t=\mathcal{T}_{k-1}+1}^{\mathcal{T}_{k-1}+T_k/2} r_t \boldsymbol{y}_t$; }
\ENDFOR
\end{algorithmic}
\end{algorithm}

\medskip

\paragraph{Regret Analysis.} Following similar arguments in Section~\ref{sec:proofthm1}, we have that:
\begin{theorem}\label{thm:context_aware}
Let $h = \min\{d,K\}$. For any $T\geq d$ and $M\geq 1$, Algorithm~\ref{alg:context-aware} may use at most $M$ batches and its regret is bounded  by 
\begin{align}
  R_T\leq  O\left(\mathrm{poly}\ln(Td)  T^{\frac{1}{2-2^{-(M-1)}}} \cdot (d\log_2(K))^{\frac{1-2^{-(M-1)}}{2-2^{-(M-1)}}}       \right) .
\end{align}
\end{theorem}

\begin{proof}
We follow the notations in  Section~\ref{sec:proofthm1}. 
Invoking Lemma~\ref{lemma:design} with $m=n = T_{k-1}$ for $2\leq k \leq M$, we have that
\begin{align}
\E\left[\min\left\{ \sqrt{\max_{\bx\in X_t^{(k-1)}}\bx^{\top}\Lambda_{k-1}^{-1}   \bx} , \sqrt{L} \right\}    \right] \leq \sqrt{\frac{1}{T_{k-1}}}\cdot O\left(\sqrt{d\ln\left(Td\right)}\right) +\frac{\sqrt{L}}{T_{k-1}}\cdot O\left(d\ln\left(Td\right)\right)  + 3\delta .\nonumber
\end{align}
As a result, the expected regret of the $k$-th batch is bounded by 
\begin{align}
   \sqrt{\ln(TKd/\delta)}\cdot  O\left(           \sqrt{\frac{T_k^2}{T_{k-1}}}\cdot O\left(\sqrt{d\ln\left(Td\right)}\right) +\frac{\sqrt{L}T_k}{T_{k-1}}\cdot O\left(d\ln\left(Td\right)\right)          +T_k\delta\right).\nonumber
\end{align}
Besides, the regret in the first batch is simply bounded by $O(T_1)$. Let $$\gamma = T^{\frac{1}{2-2^{-(M-1)}}} \cdot (d\ln(KTd)\ln(Td))^{\frac{1-2^{-(M-1)}}{2-2^{-(M-1)}}} .$$ By setting that $\delta = \frac{1}{T^3}$, $T_1 = \gamma$ and 
\begin{align}
T_k  =\sqrt{T_{k-1}}\cdot \frac{\gamma}{\sqrt{d\ln(KTd)\ln(Td)}} \label{eq:defineTnew} 
\end{align}
for $2\leq k \leq M$, we finish the proof.

\end{proof}

\subsection{The Lower Bound} 
We present a nearly matching lower bound as below.
\begin{theorem}\label{thm:lbnew}
Fix any $K\geq 2$, $T\geq d$, and any batch number $M \geq 1$. For any learning algorithm with batch complexity $M$, there exists a linear contextual bandit problem instance with dimension $d$ and $K$ arms, such that the expected regret $R_{T}$ is at least
\begin{align}
\Omega\left(      \frac{1}{\mathrm{poly}\ln(Td)} T^{\frac{1}{2-2^{-(M-1)}}} \cdot (d\ln(K))^{\frac{1-2^{-(M-1)}}{2-2^{-(M-1)}}}       \right)    . \nonumber
\end{align}
\end{theorem}
\begin{proof}[Proof sketch.]
This theorem could be proved using the same arguments as in proof of Theorem~\ref{thm:lb}. In particular, we only use the \textbf{Case \uppercase\expandafter{\romannumeral2}} construction. Below we present the value $\{T_k\}_{k=1}^M$ and $\{\epsilon_k\}_{k=1}^M$, and the detailed proof is omitted.

Define $\gamma =T^{\frac{1}{2-2^{-(M-1)}}} \cdot \left(\frac{d\log_2(K)}{M\ln(dM)}\right)^{\frac{1-2^{-(M-1)}}{2-2^{-(M-1)}}}   $. Let $T_1 = \gamma$ and $T_k = \gamma\cdot\sqrt{\frac{T_{k-1}M\ln(dM)}{d\log_2(K)}}$ for $2\leq k \leq M$. 
Let $\epsilon_1 =\frac{1}{100} $ and $\epsilon_k = \sqrt{\frac{d\log_2(K)}{M T_{k-1}\ln(dM) }}$ for $2\leq k \leq M$.  Following the analysis in Section~\ref{sec:low}, that the minimax regret is at least $\Omega(\gamma/M)$, and the conclusion follows.
\end{proof}

\section{Conclusion}\label{sec:conclusion}
In this paper, we study the batch linear contextual linear bandit problem with stochastic context. When the number of batches is limited by $M$, for any $T$ and $d$, we achieve matching upper and lower bounds for the regret (up to logarithmic factors) in both context-blind and context-aware settings. We adopt the reward-free LinUCB proposed in \citep{ruan2020linear} to achieve our learning goal. In the algorithm design and analysis, we highlight two key techniques: the scaled-and-clipped update rule and the matrix concentration inequality with dynamic upper bounds. We believe these techniques could help design and analyze batch algorithms for other online learning and decision-making problems with linear reward structures (e.g., the linear Markov Decision Processes).

\section*{Acknowledgement}
The authors would like to thank Joel A.~Tropp for the helpful discussions on matrix concentration inequalities.

\bibliographystyle{plainnat}
\bibliography{ref}

\newpage
\appendix

\section{Technical Lemmas}


\begin{lemma}[Pinsker's Inequality] \label{lemma:pinsker} Let $D_{\mathrm{KL}}(P||Q)$ denote the  KL-divergence between $P$ and $Q$. Let $D_1(P,Q):=|P-Q|_{1}$ denote the $L_1$ distance between $P$ and $Q$. 
For any two distribution $P,Q$, we have that
\begin{align}
      \int |dP -dQ|  \leq  \min\{ \sqrt{\frac{1}{2}D_{\mathrm{KL}}(P\|Q)} , \sqrt{\frac{1}{2}D_{\mathrm{KL}}(Q\|P)} \},
\end{align}
where $D_{\mathrm{KL}}$ denotes the KL-divergence.
\end{lemma}

\begin{lemma}\label{lemma:con}
Let $X_1,X_2,\ldots$ be a sequence of random variables taking value in $[0,l]$. Define $\mathcal{F}_k =\sigma(X_1,X_2,\ldots,X_{k-1})$ and $Y_k = \mathbb{E}[X_k|\mathcal{F}_k]$ for $k\geq 1$. For any $\delta>0$, we have that
\begin{align}
& \mathbb{P}\left[ \exists n, \sum_{k=1}^n X_k \leq  3\sum_{k=1}^n Y_k+ l\ln(1/\delta)\right]\leq \delta\nonumber
\\  & \mathbb{P}\left[  \exists n,  \sum_{k=1}^n Y_k \geq 3\sum_{k=1}^n X_k + l\ln(1/\delta)  \right]    \leq \delta .\nonumber 
\end{align}
\end{lemma}

\begin{proof} Let $t\in [0,1/l]$ be fixed.
Consider to bound $Z_k:=\mathbb{E}[\exp(t\sum_{k'=1}^k(X_{k'}-3Y_{k'})  )]$. By definition, we have that
\begin{align}
    \mathbb{E}[Z_k|\mathcal{F}_k]  & =\exp(t\sum_{k'=1}^{k}(X_{k'}-3Y_{k'})) \mathbb{E}\left[t(X_{k}-3Y_{k})\right] \nonumber 
    \\ & \leq \exp(t\sum_{k'=1}^{k}(X_{k'}-3Y_{k'}))\exp(3Y_{k})\cdot \mathbb{E}[1+tX_k+2t^2X^2_{k}]\nonumber 
    \\ & \leq \exp(t\sum_{k'=1}^{k}(X_{k'}-3Y_{k'}))\exp(3Y_{k})\cdot \mathbb{E}[1+3tX_k]\nonumber 
    \\ & = \exp(t\sum_{k'=1}^{k}(X_{k'}-3Y_{k'}))\exp(3Y_{k})\cdot (1+3tY_k)\nonumber 
    \\ & \leq \exp(t\sum_{k'=1}^{k}(X_{k'}-3Y_{k'}))\nonumber
    \\ & =Z_{k-1},\nonumber 
\end{align}
where the second line is by the fact that $e^x\leq 1+x+2x^2$ for $x\in [0,1]$. 
Define $Z_0 = 1$
Then $\{Z_{k}\}_{k\geq 0}$ is a super-martingale with respect to $\{\mathcal{F}_{k}\}_{k\geq 1}$. Let $\tau$ be the smallest $n$ such that $\sum_{k=1}^n X_k - 3\sum_{k=1}^n Y_k >l\ln(1/\delta)$.
It is easy to verify that $Z_{\min\{\tau,n\}}\leq \exp(tl\ln(1/\delta)+tl)<\infty$. Choose $t=1/l$.
By the optimal stopping time theorem, we have that
\begin{align}
 & \mathbb{P}\left[ \exists n\leq N, \sum_{k=1}^n X_k \geq  3\sum_{k=1}^n Y_k + l\ln(1/\delta)\right]\nonumber 
 \\ & =\mathbb{P}\left[\tau \leq N\right]\nonumber 
 \\ & \leq \mathbb{P}\left[ Z_{\min\{\tau,N\}}\geq \exp(tl\ln(1/\delta))  \right]\nonumber 
 \\ & \leq \frac{\mathbb{E}[Z_{\min\{\tau,N\}}]}{\exp(tl\ln(1/\delta))}\nonumber 
 \\ & \leq \delta. \nonumber
\end{align}

Letting $N\to \infty$, we have that
\begin{align}
& \mathbb{P}\left[ \exists n, \sum_{k=1}^n X_k \leq  3\sum_{k=1}^n Y_k+ l\ln(1/\delta)\right]\leq \delta.\nonumber
\end{align}
Considering $W_k = \mathbb{E}[\exp(t\sum_{k'=1}^k (Y_k/3-X_k))]$, using similar arguments  and choosing $t=1/(3l)$, we have that
\begin{align}
 & \mathbb{P}\left[  \exists n,  \sum_{k=1}^n Y_k \geq 3\sum_{k=1}^n X_k + l\ln(1/\delta)  \right]    \leq \delta .\nonumber 
\end{align}
The proof is completed.
\end{proof}
\section{Omitted Lemmas and Proofs in Section~\ref{sec:alg}}
\subsection{Statement and Proof of Lemma~\ref{lemma:bandit_ci}}\label{sec:pfci}

The following lemma is a similar version of Lemma 31 in \citep{ruan2020linear} that analyzes the size of the confidence interval by the ridge regression. Note that since we only assume $\|\btheta\|_\infty \leq 1$ instead of the upper bound on the Euclidean norm, the calculation is slightly different.

\begin{lemma} \label{lemma:bandit_ci}

Given $\btheta,\bx_1,\bx_2,\ldots,\bx_n\in \mathbb{R}^d$ 
such that $\|\btheta\|_{\infty}\leq 1$ for all $i\in [n]$, let $r_i=\bx_i^{\top}\btheta+\epsilon_i$ where $\{\epsilon_i\}_{i=1}^n$ are independent sub-Gaussian random variable with variance proxy $1$. Let $\Lambda = \lambda \mathbf{I}+\sum_{i=1}^n \bx_i \bx_i^{\top}$ and $\hat{\btheta}= \Lambda^{-1}\sum_{i=1}^n r_i \bx_i$. For any $\bx\in \mathbb{R}^{d}$ and any $\gamma>0$, we have that
\begin{align}
\Pr\left[|\bx^{\top}(\btheta-\hat{\btheta})| > (\gamma + \sqrt{d\lambda})\sqrt{\bx\Lambda^{-1}\bx} \right]\leq 2e^{-\frac{\gamma^2}{2}}. \nonumber
\end{align}
\end{lemma}
\begin{proof}
Direct computation gives that
\begin{align}
    |\bx^{\top}(\btheta - \hat{\btheta})| &= \left| \bx^{\top} \left( \Lambda^{-1}\sum_{i=1}^n\bx_i(\bx_i^{\top}\btheta+\epsilon_i)-\btheta \right) \right| \nonumber\\
    & = \left|\bx^{\top} \left( \Lambda^{-1}\sum_{i=1}^n \bx_i\epsilon_i + \Lambda^{-1}(\Lambda-\lambda \mathbf{I})\btheta -\btheta \right)  \right| \nonumber
    \\& = \left|\bx^{\top}\Lambda^{-1}\left( \sum_{i=1}^n \bx_i \epsilon_i - \lambda\btheta \right)\right| \nonumber
    \\ & \leq \lambda \left|\bx^{\top}\Lambda^{-1}\btheta \right|+\left| \sum_{i=1}^n \bx^{\top}\Lambda^{-1}\bx_i \epsilon_i \right|.\nonumber
\end{align}
For the first term, by Cauchy-Schwarz and noting that $\Lambda \succcurlyeq \lambda \mathbf{I}$, we have that
\begin{align}
    \lambda \left|\bx^{\top}\Lambda^{-1}\btheta \right|\leq \lambda\sqrt{d}\cdot \sqrt{\bx^{\top}\Lambda^{-2}\bx} \leq \sqrt{\lambda d \bx^{\top}\Lambda^{-1}\bx}.\label{eq:ci2} 
\end{align}
For the second term, by sub-Gaussian concentration inequalities, we have that
\begin{align}
    \mathrm{Pr}\left[  \left|\sum_{i=1}^n \bx^{\top}\Lambda^{-1}\bx_i \epsilon_i   \right| >\gamma \sqrt{\bx^{\top}\Lambda^{-1}\bx} \right] \leq 2e^{\frac{-\gamma^2}{2}}.\label{eq:ci3}
\end{align}
The conclusion follows by combining \eqref{eq:ci2} and \eqref{eq:ci3}.
\end{proof}


\subsection{Statement and Proof of the Elliptical Potential Lemma}\label{app:pf-epl}

\begin{lemma}\label{lemma:epl}
Let $\bx_{1},\bx_{2},\dots,\bx_{n}$ be a sequence of vectors in $\mathbb{R}^d$ such that $\|x_i\|_2\leq 1$. Let $\Lambda_0 = A$ be a positively definite matrix and $\Lambda_i = \Lambda_0 +\sum_{j=1}^i \bx_j\bx_j^{\top}$. If $\bx_i^{\top}\Lambda_{i-1}^{-1}\bx_i \leq 1$ for all $1\leq i \leq n$,
it then holds that
\begin{align}
    \sum_{i=1}^n \bx_i^{\top}\Lambda_{i-1}^{-1}\bx_i \leq 2\ln\left( \frac{\det(\Lambda_n)}{\det(\Lambda_0)}\right) .\nonumber
\end{align}
\end{lemma}
\begin{proof}
Note that $\det(\Lambda_{i+1})=\det(\Lambda_i) (1+\bx_{i+1}^{\top}\Lambda_i^{-1}\bx_{i+1} )$. Since $\ln(1+x)\geq \frac{x}{2}$ when $0\leq x \leq 1$, we have that
\begin{align}
    \ln(\det(\Lambda_{i+1}))- \ln(\det(\Lambda_i)) \geq \frac{1}{2}\bx_{i+1}^{\top}\Lambda_i^{-1}\bx_{i+1},\nonumber
\end{align}
which implies that 
\[  \sum_{i=1}^n \bx_i^{\top}\Lambda_{i-1}^{-1}\bx_i \leq 2 \ln\left( \frac{\det(\Lambda_n)}{\det(\Lambda_0)}\right).   \]
\end{proof}

\section{Omitted Proofs in Section~\ref{sec:pfdy}} \label{app:app-pfdy}

\subsection{Proof of Fact~\ref{fact:matrix-loenwer-order}}\label{sec:pffact}

\begin{proof}

We first show that $A \preccurlyeq B$ implies \eqref{eq:matrix-loenwer-order-2}. By $A\preccurlyeq B$, we have that $B-A \succcurlyeq 0$. Therefore, $A^{-1/2}B  A^{-1/2} - \mathbf{I}  = A^{-1/2} (B - A) A^{-1/2}$ is PSD, which means that $A^{-1/2}B  A^{-1/2}\succcurlyeq \mathbf{I}$. We then have $A^{1/2}B^{-1}  A^{1/2}\preccurlyeq \mathbf{I}$. This is because if we let $M = A^{-1/2}B  A^{-1/2} \succcurlyeq \mathbf{I}$ for convenience, for every vector $x$, it holds that \begin{align*}
    x^\top A^{1/2}B^{-1}  A^{1/2} x =  x^\top M^{-1} x 
     = (M^{-1/2} x)^\top (M^{-1/2} x) 
    \leq (M^{-1/2} x)^\top M (M^{-1/2} x) = x^\top x .
\end{align*}
This proves \eqref{eq:matrix-loenwer-order-2}, assuming $A \succcurlyeq B$.
    
We then show that \eqref{eq:matrix-loenwer-order-2} implies \eqref{eq:matrix-loenwer-order-1}. Note that   $A^{1/2}B^{-1}A^{1/2}-\mathbf{I}\preccurlyeq 0 $ implies that
$B^{-1}-A^{-1} = A^{-1/2} (A^{1/2}B^{-1}A^{1/2} - \mathbf{I}) A^{-1/2} \preccurlyeq 0 $, which leads to $B^{-1}\preccurlyeq A^{-1}$.

By symmetry, we can also prove that \eqref{eq:matrix-loenwer-order-1} implies $A \preccurlyeq B$, and therefore establishing the equivalence condition for both \eqref{eq:matrix-loenwer-order-1} and \eqref{eq:matrix-loenwer-order-2}.
    
Finally,  $A-B \preccurlyeq 0$ is equivalent to that $\mathbf{I} - B^{-1/2} AB^{-1/2} = B^{-1/2} (B - A) B^{-1/2}$ is PSD, which is also equivalent to that $B^{-1/2} AB^{-1/2}\preccurlyeq\mathbf{I}$, proving the equivalence for \eqref{eq:matrix-loenwer-order-3}.
\end{proof}

\subsection{Proof of Lemma~\ref{lemma:ac6} } \label{sec:proof-lemma-ac6}

\begin{proof}
Following the arguments in the proof of Lemma~\ref{lemma:accc4}, it suffices to prove that 
\begin{align}
    \mathbb{E}\left[\exp\left(Z_k^{-1/2}((1-\epsilon)Y_k-X_K)Z_k^{-1/2} \right) \Big|\mathcal{F}_k^+\right] \preccurlyeq \mathbf{I} \nonumber
\end{align}
holds for each $1\leq k \leq n$. Recall the definition of $U_k$ and $V_k$ in Lemma~\ref{lemma:acccc5}. In the analysis below, we will conditioned on $\mathcal{F}_k^{+}$. By the same arguments in the proof of Lemma~\ref{lemma:acccc5}, and noting that $-\mathbf{I}\preccurlyeq (1-\epsilon)V_k-U_k\preccurlyeq \mathbf{I}$, we have that
\begin{align}
    & \quad  \mathbb{E}\left[\exp\left(Z_k^{-1/2}((1-\epsilon)Y_k-X_k)Z_k^{-1/2} \right)\right] \nonumber\\ 
& =  \mathbb{E}\left[\exp\left((1-\epsilon)V_k-U_k\right)\right] \nonumber\\ 
&  = \mathbb{E}\left[ I +((1-\epsilon)V_k-U_k) +\frac{1}{2}((1-\epsilon)V_k-U_k)^2 + \sum_{i\geq 3}\frac{1}{i!} ((1-\epsilon)V_k-U_k)^i \right] \nonumber\\ 
& \leq \mathbb{E}\left[ I +((1-\epsilon)V_k-U_k) +\frac{1}{2}((1-\epsilon)V_k-U_k)^2 + \sum_{i\geq 3}\frac{1}{i!} ((1-\epsilon)V_k-U_k)^2 \right] \nonumber \\ 
& \preccurlyeq \mathbb{E}\left[\mathbf{I} +((1-\epsilon)V_k-U_k) +2((1-\epsilon)V_k-U_k)^2 \right] \nonumber\\ 
& \preccurlyeq \mathbb{E}\left[\mathbf{I} +((1-\epsilon)V_k-U_k) +4U_k^2 + 4(1-\epsilon)^2V_k^2 \right] \nonumber\\ 
& \preccurlyeq \mathbb{E}\left[\mathbf{I} +((1-\epsilon)V_k-U_k) +\frac{4\epsilon}{4(\epsilon^2+2\epsilon+2)}U_k + \frac{4(1-\epsilon)^2\epsilon}{4(\epsilon^2+2\epsilon+2)}V_k \right]\nonumber\\ 
&  \preccurlyeq \mathbf{I}.\nonumber
\end{align}
The proof is completed.
\end{proof}






\end{document}